\documentclass{article}

\usepackage[T1]{fontenc}
\usepackage[utf8]{inputenc}
\usepackage{graphicx}
\usepackage{tabularx}
\usepackage{amsmath,mathtools}
\usepackage{amssymb}
\usepackage{xcolor}
\usepackage{verbatim}
\usepackage{amsthm}
\usepackage{subcaption}
\usepackage{rotating}      
\usepackage{bbding}

\usepackage{microtype}
\usepackage{booktabs} 
\usepackage[marginal]{footmisc}
\usepackage{hyperref}

\usepackage[accepted]{icml2018}

\icmltitlerunning{First Order Generative Adversarial Networks}

\DeclareMathOperator{\supp}{supp}

\DeclareMathOperator{\oc}{OC}

\DeclareMathOperator*{\myexpectation}{\mathbb{E}}

\newtheorem{lemma}{Lemma}
\newtheorem{definition}{Definition}
\newtheorem{theorem}{Theorem}

\newtheorem{assumption}{Assumption}
\newtheorem{requirement}{Requirement}

\begin{document}

\twocolumn[
\icmltitle{First Order Generative Adversarial Networks}



\icmlsetsymbol{equal}{*}

\begin{icmlauthorlist}
\icmlauthor{Calvin Seward}{zr,jku}
\icmlauthor{Thomas Unterthiner}{jku}
\icmlauthor{Urs Bergmann}{zr}
\icmlauthor{Nikolay Jetchev}{zr}
\icmlauthor{Sepp Hochreiter}{jku}
\end{icmlauthorlist}

\icmlaffiliation{zr}{Zalando Research, M\"uhlenstraße 25, 10243 Berlin, Germany}
\icmlaffiliation{jku}{LIT AI Lab \& Institute of Bioinformatics, Johannes Kepler University Linz, Austria}

\icmlcorrespondingauthor{Calvin Seward}{calvin.seward@zalando.de}

\icmlkeywords{Machine Learning, ICML, Generative Adversarial Networks}

\vskip 0.3in
]



\printAffiliationsAndNotice{}  

\begin{abstract}
GANs excel at learning high dimensional distributions, but they can update generator parameters in directions that do not correspond to the steepest descent direction of the objective.
Prominent examples of problematic update directions include those used in both Goodfellow's original GAN and the WGAN-GP.
To formally describe an optimal update direction, we introduce a theoretical framework which allows the derivation of requirements on both the divergence and corresponding method for determining an update direction,
with these requirements guaranteeing unbiased mini-batch updates in the direction of steepest descent.
We propose a novel divergence which approximates the Wasserstein distance while regularizing the critic's first order information.
Together with an accompanying update direction, this divergence fulfills the requirements for unbiased steepest descent updates.
We verify our method, the First Order GAN, with image generation on CelebA, LSUN and CIFAR-10 and set a new state of the art on the One Billion Word language generation task.
Code to reproduce experiments is available.
\end{abstract}

 \section{Introduction}
\setlength{\footnotemargin}{5mm}
 Generative adversarial networks (GANs) \cite{goodfellow2014generative} excel at learning generative models of complex
 distributions, such as images \cite{radford2015unsupervised,ledig2016photo},
 textures \cite{jetchev2016texture,BergmannJV17,jetchev2017ganosaic}, and
 even texts \cite{gulrajani2017improved,heusel2017gans}.

 GANs learn a generative model $G$ that maps samples from multivariate random noise into a high dimensional space.
 The goal of GAN training is to update $G$ such that the generative model approximates a target probability distribution.
 In order to determine how close the generated and target distributions are, a class of divergences, the so-called adversarial
 divergences was defined and explored by \cite{liu2017approximation}. This class is broad enough to encompass most popular GAN methods
 such as the original GAN \cite{goodfellow2014generative}, $f$-GANs \cite{nowozin2016f}, moment matching networks \cite{li2015generative},
 Wasserstein GANs \cite{arjovsky2017wasserstein} and the tractable version thereof, the WGAN-GP \cite{gulrajani2017improved}.

 GANs learn a generative model with distribution $\mathbb Q$ by minimizing an objective function
 $\tau(\mathbb P\Vert\mathbb Q)$ measuring the similarity between target and generated distributions $\mathbb P$ and $\mathbb Q$.
 In most GAN settings, the objective function to be minimized is an adversarial divergence \cite{liu2017approximation}, where a critic function is learned
 that distinguishes between target and generated data.
 For example, in the classic GAN \cite{goodfellow2014generative}
 the critic $f$ classifies data as real or generated, and the generator $G$ is encouraged
 to generate samples that $f$ will classify as real.

 Unfortunately in GAN training, the generated distribution often fails to converge to the target distribution.
 Many popular GAN methods are unsuccessful with toy examples, for example failing to generate all modes of a
 mixture of Gaussians \cite{srivastava2017veegan,metz2016unrolled}
 or failing to learn the distribution of data on a one-dimensional line in a high dimensional space \cite{fedus2017many}.
 In these situations, updates to the generator don't significantly reduce the divergence between generated and target distributions;
 if there always was a significant reduction in the divergence then the generated distribution would converge to the target.

 The key to successful neural network training lies in the ability to efficiently obtain unbiased estimates of the gradients of a network's
 parameters with respect to some loss.
 With GANs, this idea can be applied to the generative setting. There, the generator $G$ is parameterized by some
 values $\theta\in\mathbb R^m$. If an unbiased estimate of the gradient of the divergence between target and generated distributions
 with respect to $\theta$ can be obtained during mini-batch learning, then SGD can be applied to learn $G$.

 In GAN learning, intuition would dictate updating the generated distribution by moving $\theta$ in the direction of steepest descent
 $-\nabla_\theta\tau(\mathbb P\Vert\mathbb Q_\theta)$. Unfortunately, $-\nabla_\theta\tau(\mathbb P\Vert\mathbb Q_\theta)$
 is generally intractable, therefore $\theta$ is updated according to a tractable method; in most cases
 a critic $f$ is learned and the gradient of the expected critic value
 $\nabla_\theta\mathbb E_{\mathbb Q_\theta}[f]$ is used as the update direction for $\theta$.
 Usually, this update direction and the direction of steepest descent $-\nabla_\theta\tau(\mathbb P\Vert\mathbb Q_\theta)$,
 don't coincide and therefore learning isn't optimal.
 As we see later, popular methods such as WGAN-GP \cite{gulrajani2017improved} are affected by this issue.

 Therefore we set out to answer a simple but fundamental question:
 Is there an adversarial divergence and corresponding method that produces unbiased estimates of the direction of steepest descent in a mini-batch setting?

 In this paper, under reasonable assumptions, we identify a path
 to such an adversarial divergence and accompanying update method. Similar to the WGAN-GP,
 this divergence also penalizes a critic's gradients, and thereby ensures that
 the critic's first order information can be used directly to obtain an update direction
 in the direction of steepest descent.

 This program places four requirements
 on the adversarial divergence and the accompanying update rule for calculating the update direction
 that haven't to the best of our knowledge been formulated together.
 This paper will give rigorous definitions of these requirements, but for now we suffice with intuitive and informal definitions:
 \begin{enumerate}
  \item[A.] the divergence used must decrease as the target and generated distributions approach each other.
  For example, if we define the trivial distance between two probability distribution to be $0$ if the distributions
  are equal, and $1$ otherwise, i.e.\
  \[\tau_{\text{trivial}}(\mathbb P\Vert\mathbb Q):=\begin{cases}0 & \mathbb P=\mathbb Q \\ 1 & \text{otherwise}\end{cases}\]
  then even as $\mathbb Q$ gets close to $\mathbb P$, $\tau_{\text{trivial}}(\mathbb P\Vert\mathbb Q)$ doesn't change.
  Without this requirement, $\nabla_\theta\tau(\mathbb P\Vert\mathbb Q_\theta)=0$ and every direction is a ``direction of steepest descent,''
  \item[B.] critic learning must be tractable,
  \item[C.] the gradient $\nabla_\theta\tau(\mathbb P\Vert\mathbb Q_\theta)$ and the result of an update rule must be well defined,
  \item[D.] the optimal critic enables an update which is an estimate of $-\nabla_\theta\tau(\mathbb P\Vert\mathbb Q_\theta)$.
 \end{enumerate}

 In order to formalize these requirements, we define in Section \ref{S:notation} the notions of adversarial divergences and
 optimal critics. In Section \ref{S:related_work} we will apply
 the adversarial divergence paradigm and begin to formalize the requirements above and better understand existing GAN methods.
 The last requirement is defined precisely in Section \ref{S:update_rule} where we explore criteria for an update rule guaranteeing a low
 variance unbiased estimate of the true gradient $-\nabla_\theta\tau(\mathbb P\Vert\mathbb Q_\theta)$.

 After stating these conditions, we devote Section \ref{S:pen_divergence} to defining
 a divergence, the Penalized Wasserstein Divergence that fulfills the first
 two basic requirements. In this setting, a critic is learned, that similarly to the WGAN-GP critic, pushes real and generated
 data as far apart as possible while being penalized if the critic violates a Lipschitz condition.

 As we will discover, an optimal critic for the Penalized Wasserstein Divergence between two distributions need not be unique.
 In fact, this divergence only specifies the values that the optimal critic assumes on the supports of generated and target distributions.
 Therefore, for many distributions, multiple critics with different gradients on the support of the generated distribution can all be optimal.

 We apply this insight in Section \ref{S:fo_divergence} and add a gradient penalty to define the
 First Order Penalized Wasserstein Divergence. This divergence enforces not just correct values for the critic, but also
 ensures that the critic's gradient, its first order information, assumes values that allow for an easy formulation of an update rule.
 Together, this divergence and update rule fulfill all four requirements.

 We hope that this gradient penalty trick will be applied to other popular GAN methods and ensure that they too return better generator updates.
 Indeed, \cite{fedus2017many} improves existing GAN methods by adding a gradient penalty.

 Finally in Section \ref{S:experiments}, the effectiveness of our method is demonstrated by generating
 images and texts.

 \section{Notation, Definitions and Assumptions}\label{S:notation}

 In \cite{liu2017approximation} an adversarial divergence is defined:
\begin{definition}[Adversarial Divergence]\label{D:adversarial_divergence_old}
 Let $X$ be a topological space, $C(X^2)$ the set of all continuous real valued functions over the Cartesian product 
 of $X$ with itself and set $\mathcal G\subseteq C(X^2)$, $\mathcal G\neq\emptyset$. An adversarial divergence $\tau(\cdot\Vert\cdot)$ over $X$ is a function
 \begin{align*}
  \mathcal P(X)\times\mathcal P(X)&\to\mathbb R\cup\{+\infty\} \\
  (\mathbb P,\mathbb Q)&\mapsto\tau(\mathbb P\Vert\mathbb Q)=\sup_{g\in\mathcal G}\mathbb E_{\mathbb P\otimes\mathbb Q}[g].
 \end{align*}
\end{definition}

The function class $\mathcal G\subseteq C(X^2)$ must be carefully selected if $\tau(\cdot\Vert\cdot)$ is to be reasonable.
For example, if $\mathcal G=C(X^2)$ then the divergence between two Dirac distributions $\tau(\delta_0\Vert\delta_1)=\infty$, and if $\mathcal G=\{0\}$, i.e.\ $\mathcal G$
contains only the constant function which assumes zero everywhere, then $\tau(\cdot\Vert\cdot)=0$.

Many existing GAN procedures can be formulated as an adversarial divergence. For example, setting
 \[\mathcal G=\{x,y\mapsto\log(u(x))+\log(1-u(y))\mid u\in\mathcal V\}\]
 \[\mathcal V=(0,1)^X\cap C(X)\footnote{$(0,1)^X$ denotes all functions mapping $X$ to $(0,1)$.}\]
results in $\tau_G(\mathbb P\Vert\mathbb Q)=\sup_{g\in\mathcal G}\mathbb E_{\mathbb P\otimes\mathbb Q}[g]$,
the divergence in Goodfellow's original GAN \cite{goodfellow2014generative}.
See \cite{liu2017approximation} for further examples.

For convenience, we'll restrict ourselves to analyzing a special case of the adversarial divergence
(similar to Theorem 4 of \cite{liu2017approximation}), and use the notation:

\begin{definition}[Critic Based Adversarial Divergence]\label{D:adversarial_divergence}
 Let $X$ be a topological space, $\mathcal F\subseteq C(X)$, $\mathcal F\neq\emptyset$.
 Further let $f\in\mathcal F$, $m_f:X\times X\to\mathbb R$, $m_f:(x,y)\mapsto m_1(f(x))-m_2(f(y))$ and $r_f\in C(X^2)$.
 Then define
 \begin{align}
  \nonumber\tau:\mathcal P(X)\times\mathcal P(X)\times\mathcal F&\to\mathbb R\cup\{+\infty\} \\
  \label{E:adversarial_divergence}(\mathbb P,\mathbb Q,f)&\mapsto\tau(\mathbb P\Vert\mathbb Q;f)=\mathbb E_{\mathbb P\otimes\mathbb Q}[m_f - r_f]
 \end{align}
 and set $\tau(\mathbb P\Vert\mathbb Q)=\sup_{f\in\mathcal F}\tau(\mathbb P\Vert\mathbb Q;f)$.
\end{definition}

For example, the $\tau_G$ from above can be equivalently defined by setting $\mathcal F=(0,1)^X\cap C(X)$, $m_1(x)=\log(x)$, $m_2(x)=\log(1-x)$ and $r_f=0$. Then
 \begin{equation}\label{E:goodfellow_gan}
 \tau_G(\mathbb P\Vert\mathbb Q)=\sup_{f\in\mathcal F}\mathbb E_{\mathbb P}[m_f-r_f]
 \end{equation}
is a critic based adversarial divergence.

An example with a non-zero $r_f$ is the WGAN-GP \cite{gulrajani2017improved}, which is a critic based adversarial divergence when $\mathcal F=C^1(X)$,
the set of all differentiable real functions on $X$,
$m_1(x)=m_2(x)=x$, $\lambda>0$ and
\[r_f(x,y)=\lambda\mathbb E_{\alpha\sim\mathcal U([0,1])}[(\Vert\nabla_z f(z)|_{\alpha x + (1-\alpha) y}\Vert-1)^2].\]
Then the WGAN-GP divergence $\tau_I(\mathbb P\Vert\mathbb Q)$ is:
\begin{equation}\label{E:WGAN_GP}
 \tau_{I}(\mathbb P\Vert\mathbb Q)=\sup_{f\in\mathcal F}\tau_{I}(\mathbb P\Vert\mathbb Q;f)=\sup_{f\in\mathcal F}\mathbb E_{\mathbb P\otimes\mathbb Q}[m_f - r_f].
\end{equation}

While Definition \ref{D:adversarial_divergence_old} is more general, Definition \ref{D:adversarial_divergence} is more in line with most GAN models.
In most GAN settings, a critic in the simpler $C(X)$ space is learned that separates real and generated data
while reducing some penalty term $r_f$ which depends on both real and generated data. For this reason, we use exclusively the notation from Definition \ref{D:adversarial_divergence}.

One desirable property of an adversarial divergence is that $\tau(\mathbb P^*\Vert\mathbb P)$ obtains its infimum if and only if
$\mathbb P^*=\mathbb P$, leading to the following definition adapted from \cite{liu2017approximation}:

\begin{definition}[Strict adversarial divergence]
 Let $\tau$ be an adversarial divergence over a topological space $X$.
 $\tau$ is called a strict adversarial divergence if for any $\mathbb P,\mathbb P^*\in\mathcal P(X)$,
 \[\tau(\mathbb P^*\Vert\mathbb P)=\inf_{\mathbb P'\in\mathcal P(X)}\tau(\mathbb P^*\Vert\mathbb P')\Rightarrow \mathbb P^*=\mathbb P\]
\end{definition}

In order to analyze GANs that minimize a critic based adversarial divergence, we introduce the set of optimal critics.
\begin{definition}[Optimal Critic,\; $\oc_{\tau}(\mathbb P,\mathbb Q)$]\label{D:opt_critic}
Let $\tau$ be a critic based adversarial divergence over a topological space $X$ and
$\mathbb P,\mathbb Q\in\mathcal P(X)$, $\mathcal F\subseteq C(X)$, $\mathcal F\neq\emptyset$.
 Define $\oc_{\tau}(\mathbb P,\mathbb Q)$ to be the set of critics in $\mathcal F$ that maximize $\tau(\mathbb P\Vert\mathbb Q;\cdot)$.
 That is
\[\oc_{\tau}(\mathbb P,\mathbb Q):=\{f\in\mathcal F\mid \tau(\mathbb P\Vert\mathbb Q;f)=\tau(\mathbb P\Vert\mathbb Q)\}.\]
\end{definition}
Note that $\oc_{\tau}(\mathbb P,\mathbb Q)=\emptyset$ is possible, \cite{arjovsky2017towards}. 
In this paper, we will always assume that if $\oc_{\tau}(\mathbb P,\mathbb Q)\not =\emptyset$, 
then an optimal critic $f^*\in\oc_{\tau}(\mathbb P,\mathbb Q)$ is known. Although is an unrealistic assumption,
see \cite{binkowski2018demystifying}, it is a good starting point for a rigorous GAN analysis.
We hope further works can extend our insights to more realistic cases of approximate critics.


Finally, we assume that generated data is distributed according to a probability distribution $\mathbb Q_\theta$
parameterized by $\theta\in\Theta\subseteq\mathbb R^m$ satisfying the mild regularity Assumption \ref{A:1}. Furthermore, we assume that $\mathbb P$ and $\mathbb Q$ both have
compact and disjoint support in Assumption \ref{A:2}. Although we conjecture that weaker assumptions can be made, we decide for the stronger assumptions
to simplify the proofs.

\begin{assumption}[Adapted from \cite{arjovsky2017wasserstein}]\label{A:1}
 Let $\Theta\subseteq\mathbb R^m$. We say $\mathbb Q_\theta\in\mathcal P(X)$, $\theta\in\Theta$
 satisfies assumption \ref{A:1} if there is a locally Lipschitz function $g:\Theta\times\mathbb R^d\to X$
 which is differentiable in the first argument
 and a distribution $\mathbb Z$ with bounded support in $\mathbb R^d$ such that
 for all $\theta\in\Theta$ it holds $\mathbb Q_\theta\sim g(\theta,z)$
 where $z\sim\mathbb Z$.
\end{assumption}

\begin{assumption}[Compact and Disjoint Distributions]\label{A:2}
 Using $\Theta\subseteq\mathbb R^m$ from Assumption \ref{A:1}, we say that $\mathbb P$ and $(\mathbb Q_\theta)_{\theta\in\Theta}$
 satisfies Assumption \ref{A:2} if for all
 $\theta\in\Theta$ it holds that the supports of
 $\mathbb P$ and $\mathbb Q_\theta$ are compact and disjoint.
\end{assumption}

\section{Requirements Derived From Related Work}\label{S:related_work}
With the concept of an Adversarial Divergence now formally defined, we can investigate
existing GAN methods from an Adversarial Divergence minimization standpoint.
 During the last few years, weaknesses in existing GAN frameworks have been highlighted and new frameworks have been
 proposed to mitigate or eliminate these weaknesses. In this section we'll trace this history and
 formalize requirements for adversarial divergences and optimal updates.

 Although using two competing neural networks for unsupervised learning isn't a new concept \cite{schmidhuber1992learning},
 recent interest in the field started when \cite{goodfellow2014generative} generated images with the divergence $\tau_G$
 defined in Eq.\ \ref{E:goodfellow_gan}.
 However, \cite{arjovsky2017towards} shows if $\mathbb P,\mathbb Q_\theta$ have compact disjoint support then
 $\nabla_\theta\tau_G(\mathbb P\Vert\mathbb Q_\theta)=0$, preventing the use of gradient based learning methods.

 In response to this impediment, the Wasserstein GAN was proposed in \cite{arjovsky2017wasserstein} with the divergence:
 \[\tau_W(\mathbb P\Vert\mathbb Q)=\sup_{\Vert f\Vert_L\leq 1}\mathbb E_{x\sim\mathbb P}[f(x)]-\mathbb E_{x'\sim\mathbb Q}[f(x')]\]
 where $\Vert f\Vert_L$ is the Lipschitz constant of $f$.
 The following example shows the advantage of $\tau_W$. Consider a series of Dirac measures $(\delta_{\frac{1}{n}})_{n>0}$.
 Then $\tau_W(\delta_0\Vert\delta_{\frac 1 n})=\frac 1 n$ while $\tau_G(\delta_0\Vert\delta_{\frac 1 n})=1$. As $\delta_{\frac 1 n}$ approaches
 $\delta_0$, the Wasserstein divergence decreases while $\tau_G(\delta_0\Vert\delta_{\frac 1 n})$ remains constant.

 This issue is explored in \cite{liu2017approximation} by creating a weak ordering, the so-called strength, of divergences.
 A divergence $\tau_1$ is said to be stronger than $\tau_2$ if for any sequence of probability measures $(\mathbb P_n)_{n\in\mathbb N}$ and
 any target probability measure $\mathbb P^*$ the convergence
 $\tau_1(\mathbb P^*\Vert\mathbb P_n)\overset{n\to\infty}{\longrightarrow}\inf_{\mathbb P\in\mathcal P(X)}\tau_1(\mathbb P^*\Vert\mathbb P)$
 implies $\tau_2(\mathbb P^*\Vert\mathbb P_n)\overset{n\to\infty}{\longrightarrow}\inf_{\mathbb P\in\mathcal P(X)}\tau_2(\mathbb P^*\Vert\mathbb P)$.
 The divergences $\tau_1$ and $\tau_2$ are equivalent if $\tau_1$ is stronger than $\tau_2$ and $\tau_2$ is stronger than $\tau_1$.
 The Wasserstein distance $\tau_W$ is the weakest divergence in the class of strict adversarial divergences \cite{liu2017approximation},
 leading to the following requirement:

 \begin{requirement}[Equivalence to $\tau_W$]\label{R:nonzero}
 An adversarial divergence $\tau$ is said to fulfill Requirement \ref{R:nonzero} if $\tau$ is a strict adversarial divergence
 which is weaker than $\tau_W$.
 \end{requirement}

 The issue of the zero gradients was side stepped in \cite{goodfellow2014generative}
 (and the option more rigorously explored in \cite{fedus2017many})
 by not updating with $-\nabla_{\theta}\mathbb E_{x'\sim\mathbb Q_\theta}[\log(1-f(x'))]$ but instead using the gradient
 $\nabla_{\theta}\mathbb E_{x'\sim\mathbb Q_\theta}[f(x')]$.
 As will be shown in Section \ref{S:update_rule}, this update direction doesn't generally move $\theta$ in the direction of steepest descent.

 Although using the Wasserstein distance as a divergence between probability measures solves many theoretical problems,
 it requires that critics are Lipschitz continuous with Lipschitz constant $1$.
 Unfortunately, no tractable algorithm has yet been found that is able to learn the optimal Lipschitz continuous critic
 (or a close approximation thereof).

 This is due in part to the fact that if the critic is
 parameterized by a neural network $f_\vartheta$, $\vartheta\in\Theta_C\subseteq\mathbb R^c$, then the set of admissible parameters
 $\{\vartheta\in\Theta_C\mid\Vert f_\vartheta\Vert_L\leq 1\}$ is highly non-convex.
 Thus critic learning is a non-convex optimization problem (as is generally the case in neural network learning)
 with non-convex constraints on the parameters.
 Since neural network learning is generally an unconstrained optimization problem, adding complex non-convex constraints makes learning intractable with current methods.
 Thus, finding an optimal Lipschitz continuous critic is a problem that can not yet be solved,
 leading to the second requirement:

 \begin{requirement}[Convex Admissible Critic Parameter Set]\label{R:convex_admissible}
 Assume $\tau$ is a critic based adversarial divergence where critics are chosen from a set $\mathcal F$.
 Assume further that in training, a parameterization
 $\vartheta\in\mathbb R^c$ of the critic function $f_\vartheta$ is learned.
 The critic based adversarial divergence $\tau$ is said to fulfill requirement \ref{R:convex_admissible} if
 the set of admissible parameters $\{\vartheta\in\mathbb R^c\mid f_\vartheta\in\mathcal F\}$ is convex.
 \end{requirement}

 It was reasoned in \cite{gulrajani2017improved} that since a Wasserstein critic must have gradients of norm at most $1$ everywhere,
 a reasonable strategy would be to transform the constrained optimization into an unconstrained optimization problem
 by penalizing the divergence when a critic has non-unit gradients.
 With this strategy, the so-called Improved Wasserstein GAN or WGAN-GP divergence defined in Eq.\ \ref{E:WGAN_GP} is obtained.

 The generator parameters are updated by training an optimal critic $f^*$ and updating with $\nabla_\theta\mathbb E_{\mathbb Q_\theta}[f^*]$.
 Although this method has impressive experimental results, it is not yet ideal.
 \cite{petzka2017regularization} showed that an optimal critic for $\tau_I$ has undefined gradients
 on the support of the generated distribution $\mathbb Q_\theta$. Thus,
 the update direction $\nabla_\theta\mathbb E_{\mathbb Q_\theta}[f^*]$ is undefined;
 even if a direction was chosen from the subgradient field (meaning the update direction is defined but random)
 the update direction won't generally point in the direction of steepest gradient descent.
 This naturally leads to the next requirement:

 \begin{requirement}[Well Defined Update Rule]\label{R:differentiable}
 An update rule is said to fulfill Requirement \ref{R:differentiable} on a target distribution $\mathbb P$
 and a family of generated distributions $(\mathbb Q_\theta)_{\theta\in\Theta}$ if
 for every $\theta\in\Theta$ the update rule at $\mathbb P$ and $\mathbb Q_\theta$ is well defined.
  \end{requirement}

 Note that kernel methods such as \cite{dziugaite2015training} and \cite{li2015generative} provide exciting theoretical guarantees
 and may well fulfill all four requirements. Since these guarantees come at a cost in scalability, we won't analyze them further.

\begin{table}
\newcommand{\nox}{\textcolor{red!70!black}{no}}
\newcommand{\yescheck}{\textcolor{green!50!black}{\checkmark}}
\centering
  \caption{Comparing existing GAN methods with regard to the four Requirements formulated in this paper.
  The methods compared are the classic GAN \cite{goodfellow2014generative},
  WGAN \cite{arjovsky2017wasserstein}, WGAN-GP \cite{gulrajani2017improved},
  WGAN-LP \cite{petzka2017regularization}, DRAGAN \cite{dragan}, PWGAN (our method) and FOGAN (our method).
  }\label{Ta:whats_what}
  \begin{tabular}{lcccc}\toprule
  & Req. \ref{R:nonzero} & Req.\ \ref{R:convex_admissible} & Req.\ \ref{R:differentiable} & Req.\ \ref{R:fo_divergence} \tabularnewline \midrule
   GAN & \nox & \yescheck & \yescheck & \nox \tabularnewline
   WGAN & \yescheck & \nox & \yescheck & \yescheck \tabularnewline
   WGAN-GP & \yescheck & \yescheck & \nox & \nox \tabularnewline
   WGAN-LP & \yescheck & \yescheck & \nox & \nox \tabularnewline
   DRAGAN & \yescheck & \yescheck & \yescheck & \nox  \tabularnewline
   PWGAN & \yescheck & \yescheck & \nox & \nox  \tabularnewline
   FOGAN & \yescheck & \yescheck & \yescheck & \yescheck \tabularnewline \bottomrule
  \end{tabular}
 \end{table}

 \section{Correct Update Rule Requirement}\label{S:update_rule}

 In the previous section, we stated a bare minimum requirement for an update rule (namely that it is well defined). In this section, we'll go further
 and explore criteria for a ``good'' update rule. For example in Lemma \ref{L:wgan_counterexample} in Section \ref{S:proof_of_things} of Appendix,
 it is shown that there exists a target $\mathbb P$ and
 a family of generated distributions $(\mathbb Q_\theta)_{\theta\in\Theta}$ fulfilling Assumptions \ref{A:1} and \ref{A:2} such that for the optimal critic
 $f_{\theta_0}^*\in\oc_{\tau_I}(\mathbb P,\mathbb Q_{\theta_0})$ there is no $\gamma\in\mathbb R$ so that
 \[\nabla_\theta\tau_I(\mathbb P\Vert\mathbb Q_\theta)|_{\theta_0}=\gamma\nabla_\theta\mathbb E_{\mathbb Q_\theta}[f_{\theta_0}^*]|_{\theta_0}\]
 for all $\theta_0\in\Theta$
 if all terms are well defined.
 Thus, the update rule used in the WGAN-GP setting, although well defined for this specific $\mathbb P$ and $\mathbb Q_{\theta_0}$,
 isn't guaranteed to move $\theta$ in the direction of steepest descent. In fact, \cite{mescheder2018which} shows that
 the WGAN-GP does not converge for specific classes of distributions.
 Therefore, the question arises what well defined update rule also moves $\theta$ in the direction of steepest descent?

 The most obvious candidate for an update rule is simply use the direction $-\nabla_\theta\tau(\mathbb P\Vert\mathbb Q_\theta)$,
 but since in the adversarial divergence setting $\tau(\mathbb P\Vert\mathbb Q_\theta)$ is the supremum over a set of infinitely many possible critics,
 calculating $-\nabla_\theta\tau(\mathbb P\Vert\mathbb Q_\theta)$ directly is generally intractable.

 One strategy to address this issue is to use an envelope theorem \cite{milgrom2002envelope}.
 Assuming all terms are well defined, then for every optimal critic $f^*\in\oc_{\tau}(\mathbb P,\mathbb Q_{\theta_0})$
 it holds $\nabla_\theta\tau(\mathbb P\Vert\mathbb Q_\theta)|_{\theta_0}=\nabla_\theta\tau(\mathbb P\Vert\mathbb Q_\theta;f^*)|_{\theta_0}$.
 This strategy is outlined in detail in \cite{arjovsky2017wasserstein} when proving the Wasserstein GAN update rule,
 and explored in the context of the classic GAN divergence $\tau_G$ in \cite{arjovsky2017towards}.

 Yet in many GAN settings,
 \cite{goodfellow2014generative,arjovsky2017wasserstein,salimans2016improved,petzka2017regularization},
 the update rule is to train an optimal critic $f^*$ and then take a step in the direction of $\nabla_\theta\mathbb E_{\mathbb Q_\theta}[f^*]$.
 In the critic based adversarial divergence setting (Definition \ref{D:adversarial_divergence}), a direct result of Eq.\ \ref{E:adversarial_divergence}
 together with Theorem 1 from \cite{milgrom2002envelope} is that for every $f^*\in\oc_{\tau}(\mathbb P,\mathbb Q_{\theta_0})$
 \begin{align}
 \nonumber\nabla_\theta\tau(\mathbb P\Vert\mathbb Q_\theta)|_{\theta_0}
 &=\nabla_\theta\tau(\mathbb P\Vert\mathbb Q;f^*)\\
 \label{E:update_rule}&=-\nabla_\theta(\mathbb E_{\mathbb Q_\theta}[m_2(f^*)]+\mathbb E_{\mathbb P\otimes\mathbb Q_\theta}[r_{f^*}])|_{\theta_0}
 \end{align}
 when all terms are well defined. Thus, the update direction $\nabla_\theta\mathbb E_{\mathbb Q_\theta}[f^*]$ only points in the direction of steepest descent for special choices of $m_2$ and $r_f$.
 One such example is the Wasserstein GAN where $m_2(x)=x$ and $r_f=0$.

 Most popular GAN methods don't employ functions $m_2$ and $r_f$ such that the update direction $\nabla_\theta\mathbb E_{\mathbb Q_\theta}[f^*]$ points in the direction of steepest descent.
 For example, with the classic GAN, $m_2(x)=\log(1-x)$ and $r_f=0$, so the update direction $\nabla_\theta\mathbb E_{\mathbb Q_\theta}[f^*]$ clearly is not oriented
 in the direction of steepest descent $\nabla_\theta\mathbb E_{\mathbb Q_\theta}[\log(1-f^*)]$.
 The WGAN-GP is similar, since as we see in Lemma \ref{L:wgan_counterexample} in Appendix, Section \ref{S:proof_of_things},
 $\nabla_\theta\mathbb E_{\mathbb P\otimes\mathbb Q_\theta}[r_{f^*}]$ is not generally
 a multiple of $\nabla_\theta\mathbb E_{\mathbb Q_\theta}[f^*]$.

 The question arises why this direction is used instead of directly calculating the direction of steepest descent?
 Using the correct update rule in Eq.\ \ref{E:update_rule} above involves estimating 
 $\nabla_\theta\mathbb E_{\mathbb P\otimes\mathbb Q_\theta}[r_{f^*}]$, which requires sampling from both $\mathbb P$ and $\mathbb Q_\theta$.
 GAN learning happens in mini-batches, therefore $\nabla_\theta\mathbb E_{\mathbb P\otimes\mathbb Q_\theta}[r_{f^*}]$ isn't calculated directly, but estimated
 based on samples which can lead to variance in the estimate.

 To analyze this issue, we use the notation from \cite{bellemare2017cramer} where $\mathbf{X}_m := X_1, X_2,\ldots, X_m$ are samples from $\mathbb P$ and
 the empirical distribution $\hat{\mathbb P}_m$ is defined by
$\hat{\mathbb P}_m := \hat{\mathbb P}_m(\mathbf{X}_m) := \frac 1 m \sum_{i=1}^m\delta_{X_i}$.
Further let $\mathbb V_{\mathbf{X}_m\sim\mathbb P}$ be the element-wise variance.
Now with mini-batch learning we get\footnote{Because the first expectation doesn't depend on $\theta$, $\nabla_\theta\mathbb E_{\hat{\mathbb P}_m}[m_1(f^*)]=0$. In the same way,
 because the second expectation doesn't depend on the mini-batch $\mathbf X_m$ sampled,  $\mathbb V_{\mathbf{X}_m\sim\mathbb P}[\mathbb E_{\mathbb Q_\theta}[m_2(f^*)]]=0$.}
\begin{align*}
 &\,\mathbb V_{\mathbf{X}_m\sim\mathbb P}[\nabla_\theta\mathbb E_{\hat{\mathbb P}_m\otimes\mathbb Q_\theta}[m_{f^*}-r_{f^*}]|_{\theta_0}] \\
 =\,&\,\mathbb V_{\mathbf{X}_m\sim\mathbb P}[\nabla_\theta(\mathbb E_{\hat{\mathbb P}_m}[m_1(f^*)]-\mathbb E_{\mathbb Q_\theta}[m_2(f^*)] \\
 &\quad\quad\quad\;\;\;\,-\mathbb E_{\hat{\mathbb P}_m\otimes\mathbb Q_\theta}[r_{f^*}])|_{\theta_0}]\\
 =\,&\,\mathbb V_{\mathbf{X}_m\sim\mathbb P}[\nabla_\theta\mathbb E_{\hat{\mathbb P}_m\otimes\mathbb Q_\theta}[r_{f^*}]|_{\theta_0}].
\end{align*}
Therefore, estimation of $\nabla_\theta\mathbb E_{\mathbb P\otimes\mathbb Q_\theta}[r_{f^*}]$ is an extra source of variance.

Our solution to both these problems chooses the critic based adversarial divergence $\tau$
in such a way that there exists a $\gamma\in\mathbb R$
so that for all optimal critics $f^*\in\oc_{\tau}(\mathbb P,\mathbb Q_{\theta_0})$ it holds
\begin{equation}\label{E:first_order_assumption}
 \nabla_\theta\mathbb E_{\mathbb P\otimes\mathbb Q_\theta}[r_{f^*}]|_{\theta_0}\approx \gamma\nabla_\theta\mathbb E_{\mathbb Q_\theta}[m_2(f^*)]|_{\theta_0}.
\end{equation}
In Theorem \ref{T:fogan_theorem} we see conditions on $\mathbb P, \mathbb Q_\theta$ such that equality holds. Now using Eq.\ \ref{E:first_order_assumption} we see that
\begin{align*}
\nabla_\theta\tau(\mathbb P\Vert\mathbb Q_\theta)|_{\theta_0}&=-\nabla_\theta(\mathbb E_{\mathbb Q_\theta}[m_2(f^*)]+\mathbb E_{\mathbb P\otimes\mathbb Q_\theta}[r_{f^*}])|_{\theta_0}\\
&\approx-\nabla_\theta(\mathbb E_{\mathbb Q_\theta}[m_2(f^*)]+\gamma\mathbb E_{\mathbb Q_\theta}[m_2(f^*)])|_{\theta_0}\\
&=-(1+\gamma)\nabla_\theta\mathbb E_{\mathbb Q_\theta}[m_2(f^*)]|_{\theta_0}
\end{align*}
making $(1+\gamma)\nabla_\theta\mathbb E_{\mathbb Q_\theta}[m_2(f^*)]$ a low variance update approximation of the direction of steepest descent.

We're then able to have the best of both worlds.
On the one hand, when $r_f$ serves as a penalty term, training of a critic neural network can happen in
an unconstrained optimization fashion like with the WGAN-GP.
At the same time, the direction of steepest descent can be approximated by calculating $\nabla_\theta\mathbb E_{\mathbb Q_\theta}[m_2(f^*)]$,
and as in the Wasserstein GAN we get reliable gradient update steps.

With this motivation, Eq.\ \ref{E:first_order_assumption} forms the basis of our final requirement:
\begin{requirement}[Low Variance Update Rule]\label{R:fo_divergence}
 An adversarial divergence $\tau$ is said to fulfill requirement \ref{R:fo_divergence} if
 $\tau$ is a critic based adversarial divergence and
 for every optimal critic $f^*\in\oc_{\tau}(\mathbb P,\mathbb Q_{\theta_0})$ fulfills Eq.\ \ref{E:first_order_assumption}.
 \end{requirement}
 
It should be noted that the WGAN-GP achieves impressive experimental results; we conjecture that in
many cases $\nabla_\theta\mathbb E_{\mathbb Q_\theta}[f^*]$ close enough to the true direction of steepest descent.
Nevertheless, as the experiments in Section \ref{S:experiments} show, our gradient estimates lead to better convergence in
a challenging language modeling task.

\section{Penalized Wasserstein Divergence}\label{S:pen_divergence}
 We now attempt to find an adversarial
 divergence that fulfills all four requirements. We start by formulating an adversarial divergence $\tau_P$ and a corresponding update rule
 than can be shown to comply with Requirements \ref{R:nonzero} and \ref{R:convex_admissible}.
 Subsequently in Section \ref{S:fo_divergence}, $\tau_P$ will be refined to make its update rule
 practical and conform to all four requirements.

 The divergence $\tau_P$ is inspired by the Wasserstein distance, there for an optimal critic between two Dirac distributions
 $f^*\in\oc_{\tau}(\delta_a,\delta_b)$ it holds $f(a)-f(b)=|a-b|$. Now if we look at
 \begin{equation}\label{E:simple_pwgan}
  \tau_{\text{simple}}(\delta_a\Vert\delta_b):=\sup_{f\in\mathcal F}f(a)-f(b)-\frac{(f(a)-f(b))^2}{|a-b|}
 \end{equation}
 it's easy to calculate that $\tau_{\text{simple}}(\delta_a\Vert\delta_b)=\frac{1}{4}|a-b|$, which is the same up to a constant (in this simple setting)
 as the Wasserstein distance, without being a constrained optimization problem. See Figure \ref{F:example} for an example.

 This has another intuitive explanation. Because Eq.\ \ref{E:simple_pwgan} can be reformulated as
 \[\tau_{\text{simple}}(\delta_a\Vert\delta_b)=\sup_{f\in\mathcal F}f(a)-f(b)-|a-b|\left(\frac{f(a)-f(b)}{|a-b|}\right)^2\]
 which is a tug of war between the objective $f(a)-f(b)$ and the squared Lipschitz penalty $\frac{|f(a)-f(b)|}{|a-b|}$ weighted by $|a-b|$.
 This $|a-b|$ term is important (and missing from \cite{gulrajani2017improved}, \cite{petzka2017regularization})
 because otherwise the slope of the optimal critic between $a$ and $b$ will depend on $|a-b|$.

 The penalized Wasserstein divergence $\tau_P$ is a straight-forward adaptation of $\tau_{\text{simple}}$ to the multi dimensional case.
 \begin{definition}[Penalized Wasserstein Divergence]\label{D:pen_divergence}
  Assume $X\subseteq\mathbb R^n$ and $\mathbb P,\mathbb Q\in\mathcal P(X)$ are probability measures over $X,$ $\lambda>0$
  and $\mathcal F=C^1(X)$. Set
  \begin{align*}
   \tau_P(\mathbb P\Vert\mathbb Q;f):=\,&\mathbb E_{x\sim\mathbb P}[f(x)]-\mathbb E_{x'\sim\mathbb Q}[f(x')] \\
  -\lambda\,&\mathbb E_{x\sim\mathbb P,x'\sim\mathbb Q}\left[\frac{(f(x)-f(x'))^2}{\Vert x-x'\Vert}\right].
   \end{align*}
  Define the penalized Wasserstein divergence as
  \[\tau_P(\mathbb P\Vert\mathbb Q)=\sup_{f\in\mathcal F}\tau_P(\mathbb P\Vert\mathbb Q;f).\]
  This divergence is updated by picking an optimal critic $f^*\in\oc_{\tau_P}(\mathbb P,\mathbb Q_{\theta_0})$
  and taking a step in the direction of $\nabla_\theta\mathbb E_{\mathbb Q_\theta}[f^*]|_{\theta_0}$.
 \end{definition}

 This formulation is similar to the WGAN-GP \cite{gulrajani2017improved}, restated here in Eq.\ \ref{E:WGAN_GP}.
 \begin{theorem}\label{T:pwgan_theorem}
  Assume $X\subseteq\mathbb R^n$, and $\mathbb P,\mathbb Q_\theta\in\mathcal P(X)$ are probability measures over
  $X$ fulfilling Assumptions \ref{A:1} and \ref{A:2}.
  Then for every $\theta_0\in\Theta$ the Penalized Wasserstein Divergence with it's corresponding update direction
  fulfills Requirements \ref{R:nonzero} and \ref{R:convex_admissible}.
  
  Further, there exists an optimal critic $f^*\in\oc_{\tau_P}(\mathbb P,\mathbb Q_{\theta_0})$ that
  fulfills Eq.\ \ref{E:first_order_assumption} and thus Requirement \ref{R:fo_divergence}.
 \end{theorem}
 \begin{proof}
  See Appendix, Section \ref{S:proof_of_things}.
 \end{proof}

 Note that this theorem isn't unique to $\tau_P$.
 For example, for the penalty in Eq.\ 8 of \cite{petzka2017regularization} we conjecture that a similar result can be shown.
 The divergence $\tau_P$ is still very useful because, as will be shown in the next section,
 $\tau_P$ can be modified slightly to obtain a new critic $\tau_F$, where every optimal critic
 fulfills Requirements \ref{R:nonzero} to \ref{R:fo_divergence}.
 
 Since $\tau_P$ only constrains the value of a critic on the supports of $\mathbb P$ and $\mathbb Q_\theta$, 
 many different critics are optimal, and in general $\nabla_\theta\mathbb E_{\mathbb Q_\theta}[f^*]$ depends on the optimal critic choice and is thus is not well defined.
 With this, Requirements \ref{R:differentiable} and \ref{R:fo_divergence} are not fulfilled.
 See Figure \ref{F:example} for a simple example.
 
 \begin{figure}
\begin{minipage}[b]{.5\linewidth}
\centering
\includegraphics[width=\linewidth]{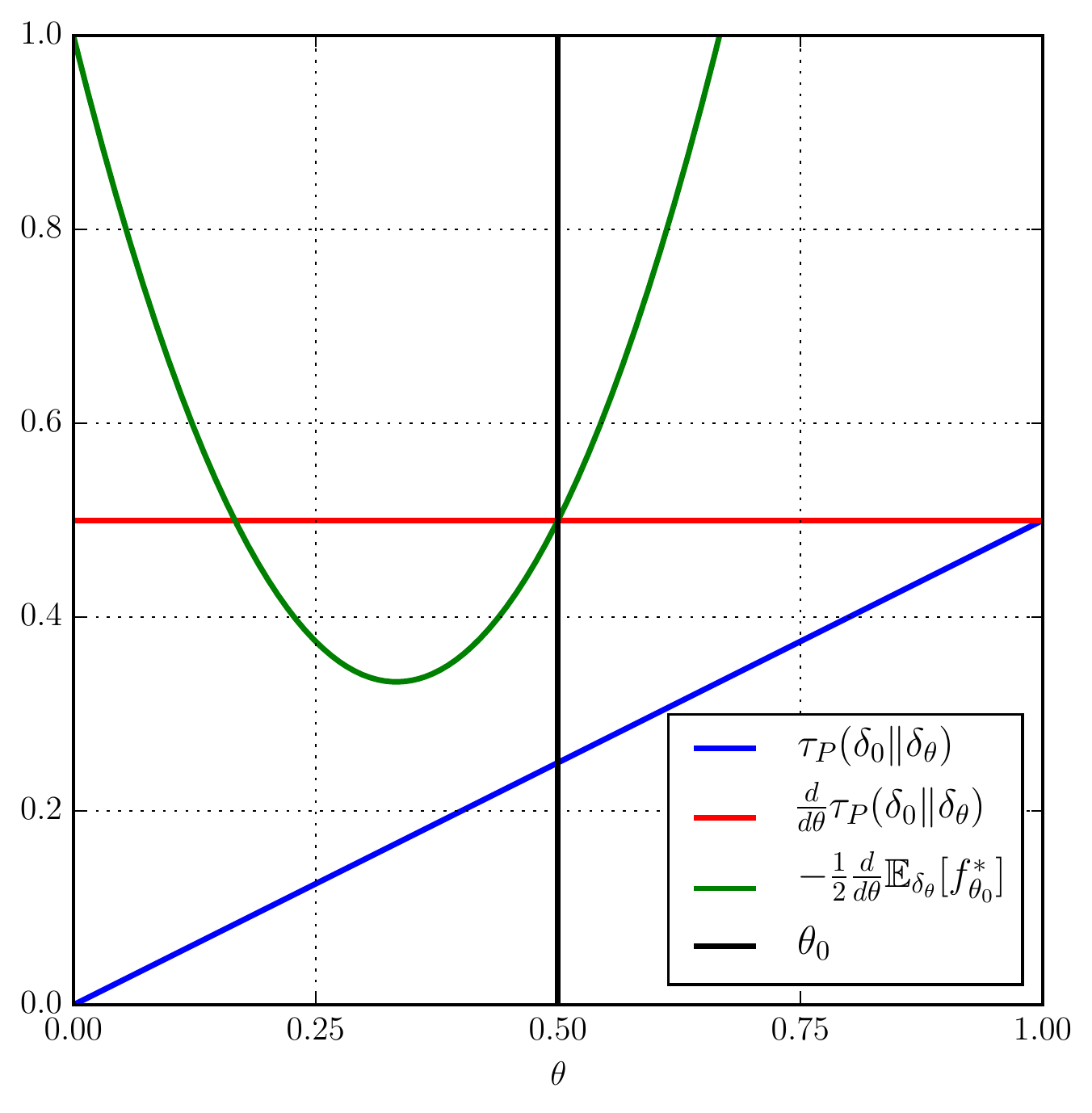}
\subcaption{\footnotesize{First order critic}}\label{F:first_order_critic}
\end{minipage}%
\begin{minipage}[b]{.5\linewidth}
\centering
\includegraphics[width=\linewidth]{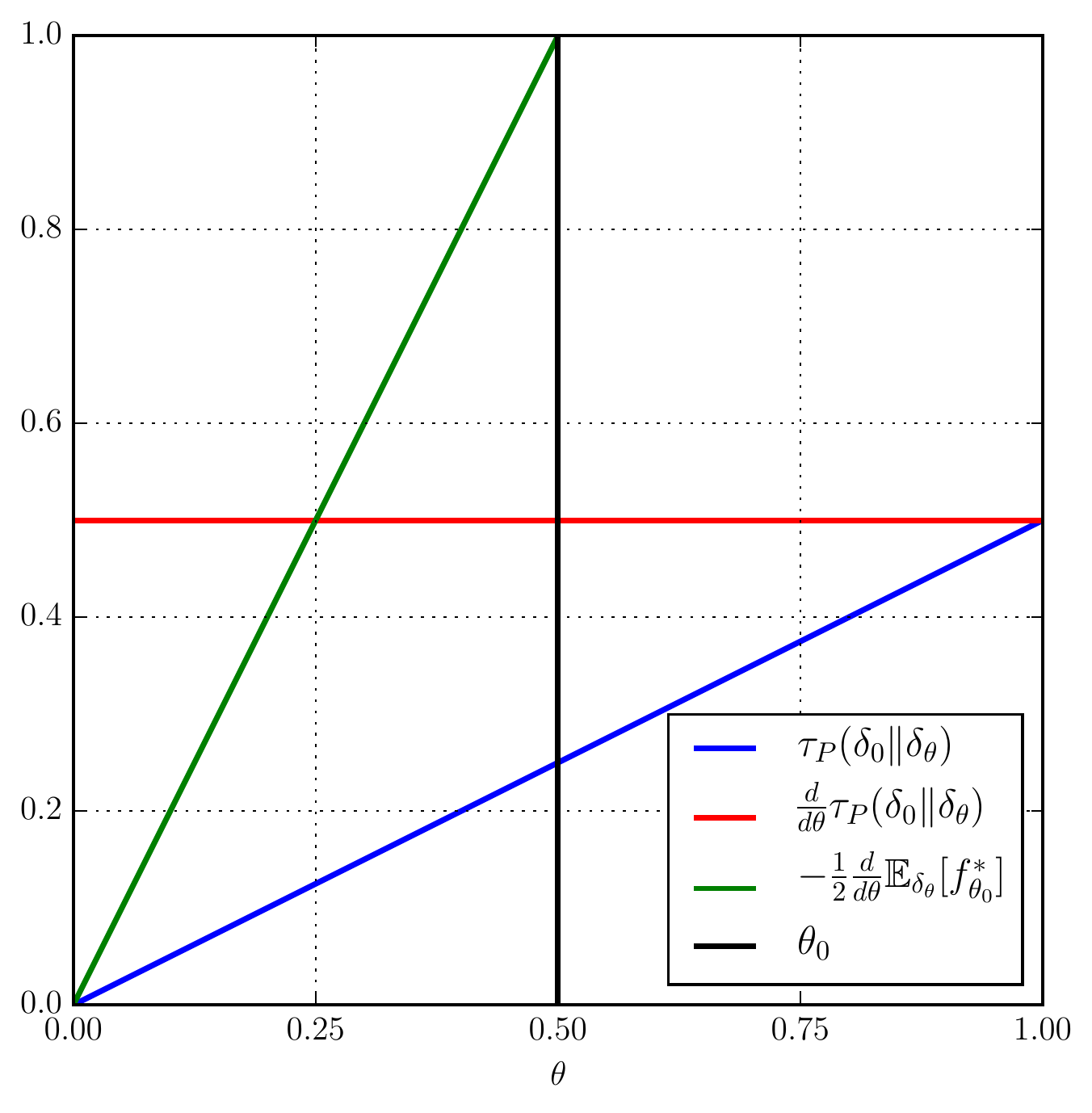}
\subcaption{\footnotesize{Normal critic}}\label{F:normal_critic}
\end{minipage}
\caption{
Comparison of $\tau_P$ update rule given different optimal critics.
Consider the simple example of divergence $\tau_P$ from Definition \ref{D:pen_divergence} between Dirac measures 
with update rule $\frac 1 2 \frac{d}{d\theta}\mathbb E_{\delta_\theta}[f]$
(the update rule is from Lemma \ref{L:first_order} in Appendix, Section \ref{S:proof_of_things}).
Recall that $\tau_P(\delta_0\Vert\delta_\theta;f)=-f(\theta)-\frac{(f(\theta))^2}{2\theta}$, and that so $\frac{d}{d\theta}\tau_P(\delta_0\Vert\delta_\theta)=\frac 1 2$.
Let $\theta_0=0.5$; our goal is to calculate $\frac{d}{d\theta}\tau_P(\delta_0\Vert\delta_\theta)|_{\theta=\theta_0}$ via our update rule.
Since multiple critics are optimal for $\tau_P(\delta_0\Vert\delta_\theta)$, we will explore how the choice of optimal critic affects the update.
In Subfigure \ref{F:first_order_critic}, we chose the first order optimal critic $f^*_{\theta_0}(x)=x(-4x^2 + 4x - 2)$, and
$\frac{d}{d\theta}\tau_P(\delta_0\Vert\delta_\theta)|_{\theta=\theta_0}=-\frac 1 2 \frac{d}{d\theta}\mathbb E_{\delta_\theta}[f^*_{\theta_0}]_{\theta=\theta_0}$
and the update rule is correct (see how the red, black and green lines all intersect in one point). In Subfigure \ref{F:normal_critic},
the optimal critic is set to $f^*_{\vartheta_0}(x)=-2x^2$ which is not a first order critic
resulting in the update rule calculating an incorrect update.
}\label{F:example}
\end{figure}
 
 In theory, $\tau_P$'s critic could be trained with a modified sampling procedure so that $\nabla_\theta\mathbb E_{\mathbb Q_\theta}[f^*]$
 is well defined and Eq.\ \ref{E:first_order_assumption} holds, as is done in both \cite{dragan} and \cite{unterthiner2017coulomb}.
 By using a method similar to \cite{bishop1998gtm}, one can minimize the divergence $\tau_P(\mathbb P,\hat {\mathbb Q}_\theta)$
 where $\hat {\mathbb Q}_\theta$ is data equal to $x'+\epsilon$ where $x'$ is sampled from $\mathbb Q_\theta$ and $\epsilon$ is some zero-mean
 uniform distributed noise. In this way the support
 of $\hat {\mathbb Q}_\theta$ lives in the full space $X$ and not the submanifold $\supp(\mathbb Q_\theta)$.
 Unfortunately, while this method works in theory,
 the number of samples required for accurate gradient estimates scales with the dimensionality of the underlying space $X$,
 not with the dimensionality of data or generated submanifolds $\supp(\mathbb P)$ or $\supp(\mathbb Q_\theta)$.
 In response, we propose the First Order Penalized Wasserstein Divergence.
 
 \section{First Order Penalized Wasserstein Divergence}\label{S:fo_divergence}
 As was seen in the last section, since $\tau_P$ only constrains the value of optimal critics on the supports of $\mathbb P$ and $\mathbb Q_\theta$,
 the gradient $\nabla_\theta\mathbb E_{\mathbb Q_\theta}[f^*]$ is not well defined. A natural method to refine $\tau_P$ to achieve
 a well defined gradient is to enforce two things:
 \begin{itemize}
  \item $f^*$ should be optimal on a larger manifold, namely the manifold $\mathbb Q'_\theta$ that is created by
 ``stretching'' $\mathbb Q_\theta$ bit in the direction of $\mathbb P$ (the formal definition is below). 
  \item The norm of the gradient of the optimal critic, $\Vert\nabla_x f^*(x) \Vert$ on $\supp(\mathbb Q'_\theta)$ should be equal to the norm of the maximal directional
  derivative in the support of $\mathbb Q'_\theta$ (see Eq.\ \ref{E:g_zero} in Appendix).
 \end{itemize}
 By enforcing these two points, we assure that $\nabla_x f^*(x)$ is well defined and points towards the real data $\mathbb P$. Thus, the following
 definition emerges (see proof of Lemma \ref{L:first_order} in Appendix, Section \ref{S:proof_of_things} for details).
 \begin{definition}[First Order Penalized Wasserstein Divergence (FOGAN)]\label{D:fo_divergence}
  Assume $X\subseteq \mathbb R^n$ and $\mathbb P,\mathbb Q\in\mathcal P(X)$ are probability measures over $X$.
  Set $\mathcal F=C^1(X)$, $\lambda,\mu>0$ and
  \begin{small}
  \begin{align*}
   &\tau_F(\mathbb P\Vert\mathbb Q;f):=\mathbb E_{x\sim\mathbb P}[f(x)]-\mathbb E_{x'\sim\mathbb Q}[f(x')]\\
  &-\lambda\myexpectation_{x\sim\mathbb P,x'\sim\mathbb Q}\left[\frac{(f(x)-f(x'))^2}{\Vert x-x'\Vert}\right] \\
  &-\mu\myexpectation_{x\sim\mathbb P, x'\sim\mathbb Q}\left(\Vert\nabla_x f(x)\big|_{x'}\Vert-
  \frac{\left\Vert\mathbb E_{\tilde x\sim\mathbb P}[(\tilde x- x')\frac{f(\tilde x)-f(x')}{\Vert x'-\tilde x\Vert^3}]\right\Vert}
 {\mathbb E_{\tilde x\sim\mathbb P}[\frac{1}{\Vert x'-\tilde x\Vert}]}
  \right)^2
  \end{align*}
  \end{small}
  Define the First Order Penalized Wasserstein Divergence as
  \[\tau_F(\mathbb P\Vert\mathbb Q)=\sup_{f\in\mathcal F}\tau_F(\mathbb P\Vert\mathbb Q;f).\]
  This divergence is updated by picking an optimal critic $f^*\in\oc_{\tau_P}(\mathbb P,\mathbb Q_{\theta_0})$
  and taking a step in the direction of $\nabla_\theta\mathbb E_{\mathbb Q_\theta}[f^*]|_{\theta_0}$.
 \end{definition}

 In order to define a GAN from the First Order Penalized Wasserstein Divergence, we must define a slight modification of
 the generated distribution $\mathbb Q_\theta$ to obtain $\mathbb Q'_\theta$. Similar to the WGAN-GP setting,
 samples from $\mathbb Q'_\theta$ are obtained by $x'-\alpha(x'-x)$ where $x\sim\mathbb P$ and $x'\sim\mathbb Q_\theta$.
 The difference is that $\alpha\sim\mathcal U([0,\varepsilon])$, with $\varepsilon$ chosen small, making $\mathbb Q_\theta$
 and $\mathbb Q'_\theta$ quite similar. Therefore updates to $\theta$ that reduce $\tau_F(\mathbb P\Vert\mathbb Q'_\theta)$ also reduce
 $\tau_F(\mathbb P\Vert\mathbb Q_\theta)$.
 
 Conveniently, as is shown in Lemma \ref{L:adversary_subset} in Appendix, Section \ref{S:proof_of_things},
 any optimal critic for the First Order Penalized Wasserstein divergence
 is also an optimal critic for the Penalized Wasserstein Divergence.  
 The key advantage to the First Order Penalized Wasserstein Divergence is that for any $\mathbb P$, $\mathbb Q_\theta$ fulfilling Assumptions \ref{A:1}
 and \ref{A:2}, $\tau_F(\mathbb P\Vert\mathbb Q'_\theta)$ with its corresponding update rule
 $\nabla_\theta\mathbb E_{\mathbb Q'_\theta}[f^*]$ on the slightly modified probability distribution $\mathbb Q'_\theta$
  fulfills requirements \ref{R:differentiable} and \ref{R:fo_divergence}.

 \begin{theorem}\label{T:fogan_theorem}
  Assume $X\subseteq\mathbb R^n$, and $\mathbb P,\mathbb Q_\theta\in\mathcal P(X)$ are probability measures over
  $X$ fulfilling Assumptions \ref{A:1} and \ref{A:2} and $\mathbb Q'_\theta$ is $\mathbb Q_\theta$ modified
  using the method above.
  Then for every $\theta_0\in\Theta$ there exists at least one optimal critic $f^*\in\oc_{\tau_F}(\mathbb P,\mathbb Q'_{\theta_0})$
  and $\tau_F$ combined with update direction $\nabla_\theta\mathbb E_{\mathbb Q_\theta}[f^*]|_{\theta_0}$ fulfills 
  Requirements \ref{R:nonzero} to \ref{R:fo_divergence}. If $\mathbb P,\mathbb Q'_\theta$ are such that $\forall x,x'\in\supp(\mathbb P),\supp(\mathbb Q'_\theta)$
  it holds $f^*(x)-f^*(x')=c\Vert x-x'\Vert$ for some constant $c$, then equality holds for Eq.\ \ref{E:first_order_assumption}.
 \end{theorem}
 \begin{proof}
  See Appendix, Section \ref{S:proof_of_things}
 \end{proof}

 Note that adding a gradient penalty, other than being a necessary step for the WGAN-GP \cite{gulrajani2017improved}, 
 DRAGAN \cite{dragan} and Consensus Optimization GAN \cite{mescheder2017numerics},
 has also been shown empirically to improve the performance the original GAN method (Eq.\ \ref{E:goodfellow_gan}),
 see \cite{fedus2017many}. In addition, using stricter assumptions on the critic, \cite{nagarajan2017gradient} provides a theoretical justification for 
 use of a gradient penalty in GAN learning.
 The analysis of Theorem \ref{T:fogan_theorem} in Appendix, Section \ref{S:proof_of_things} provides a theoretical understanding
 why in the Penalized Wasserstein GAN setting adding a gradient penalty causes $\nabla_\theta\mathbb E_{\mathbb Q'_\theta}[f^*]$
 to be an update rule that points in the direction of steepest descent, and may provide a path for other GAN methods to make similar assurances.

 \section{Experimental Results}\label{S:experiments}

 \subsection{Image Generation}
 We begin by testing the FOGAN on the CelebA image generation task \cite{liu2015deep},
 training a generative model with the DCGAN  architecture \cite{radford2015unsupervised} and obtaining Fr\'echet Inception Distance (FID) scores \cite{heusel2017gans}
 competitive with state of the art methods without doing a tuning parameter search.
 Similarly, we show competitive results on LSUN \cite{yu15lsun} and CIFAR-10 \cite{krizhevsky2009learning}.
 See Table \ref{Ta:results}, Appendix \ref{SS:appendix_celeba_words} and released code
 \footnote{\scriptsize{\url{https://github.com/zalandoresearch/first_order_gan}}}.

 \subsection{One Billion Word}
 Finally, we use the First Order Penalized Wasserstein Divergence to train a character level generative language model on the
 One Billion Word Benchmark \cite{chelba2013one}. In this setting, a 1D CNN deterministically transforms a latent
 vector into a $32\times C$ matrix, where $C$ is the number of possible characters.
 A softmax nonlinearity is applied to this output, and given to the critic.
 Real data is one-hot encoding of 32 character texts sampled from the true data.
 
  \begin{table}
\caption{Comparison of different GAN methods for image and text generation.
We measure performance with respect to the FID on the image datasets
and JSD between $n$-grams for text generation.}\label{Ta:results}
\centering
\begin{scriptsize}
  \begin{tabular}{cccccc}\toprule
  Task & BEGAN & DCGAN & Coulomb & WGAN-GP & FOGAN \\ \midrule
  CelebA & 28.5 & 12.5 & 9.3 & 4.2 & 6.0 \\
  LSUN & 112 & 57.5 & 31.2 &  9.5 & 11.4 \\
  CIFAR-10 & - & - & 27.3 & 24.8 & 27.4 \\
  4-gram & -    & -    & -   & $.220\pm .006$ & $.226\pm .006$ \\
  6-gram & - & - & -   & $.573\pm .009$ & $.556 \pm .004$ \\ \bottomrule
  \end{tabular}
 \end{scriptsize}
 \end{table}

 We conjecture this is an especially difficult task for GANs, since data in the target distribution lies in just a few corners of
 the $32\times C$ dimensional unit hypercube. As the generator is updated, it must push mass from one corner to another, passing through
 the interior of the hypercube far from any real data. Methods other than Coulomb GAN \cite{unterthiner2017coulomb}
 WGAN-GP \cite{gulrajani2017improved,heusel2017gans} and the Sobolev GAN \cite{mroueh2017sobolev}
 have not been shown to be successful at this task.
 
 We use the same setup as in both \cite{gulrajani2017improved,heusel2017gans}
 with two differences.
 First, we train to minimize our divergence
 from Definition \ref{D:fo_divergence} with parameters $\lambda=0.1$ and $\mu=1.0$ instead of the WGAN-GP divergence.
 Second, we use batch normalization in the generator, both for training our FOGAN method and the benchmark WGAN-GP;
 we do this because batch normalization improved performance and stability of both models.
 
 As with \cite{gulrajani2017improved,heusel2017gans} we use the Jensen-Shannon-divergence (JSD) between $n$-grams from the model and the real world distribution
 as an evaluation metric.
 The JSD is estimated by sampling a finite number of 32 character vectors,
 and comparing the distributions of the $n$-grams from said samples and true data.
 This estimation is biased; smaller samples result in larger JSD estimations.
 A Bayes limit results from this bias;
 even when samples are drawn from real world data and compared with real world data, small sample sizes results in large JSD estimations.
 In order to detect performance difference when training with the FOGAN and WGAN-GP, a low Bayes limit is necessary.
 Thus, to compare the methods, we sampled $6400$ 32 character vectors in contrast with
 the $640$ vectors sampled in past works. Therefore, the JSD values in those papers are higher than the results here.
 
 For our experiments we trained both models for $500,000$ iterations in $5$ independent runs, estimating the JSD between $6$-grams
 of generated and real world data every $2000$ training steps, see Figure \ref{F:jsd_6}.
 The results are even more impressive when
 aligned with wall-clock time. Since in WGAN-GP training an extra point between real and generated distributions must be sampled, it is slower
 than the FOGAN training; see Figure \ref{F:jsd_6}
 and observe the significant ($2\sigma$) drop in estimated JSD.
 
 \begin{figure}
\begin{minipage}{\linewidth}
\centering
\includegraphics[width=1\linewidth]{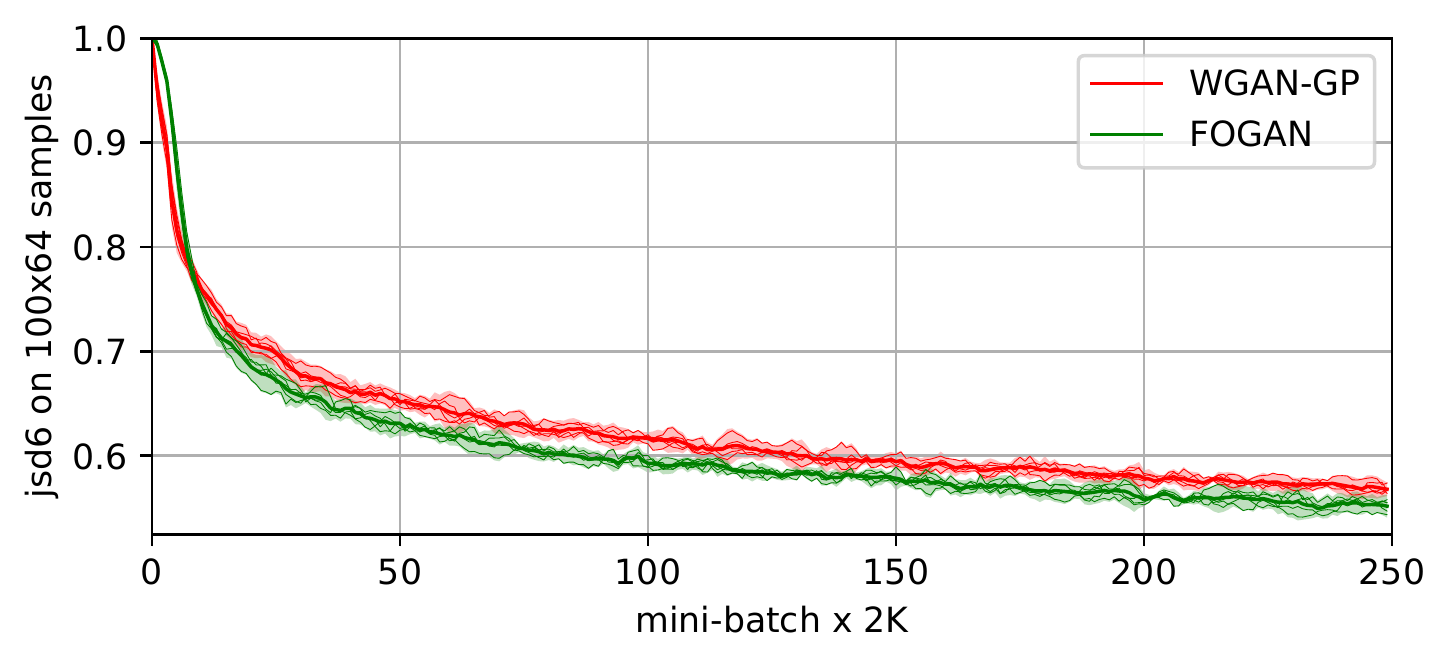}
\end{minipage}
\begin{minipage}{\linewidth}
\centering
\includegraphics[width=1\linewidth]{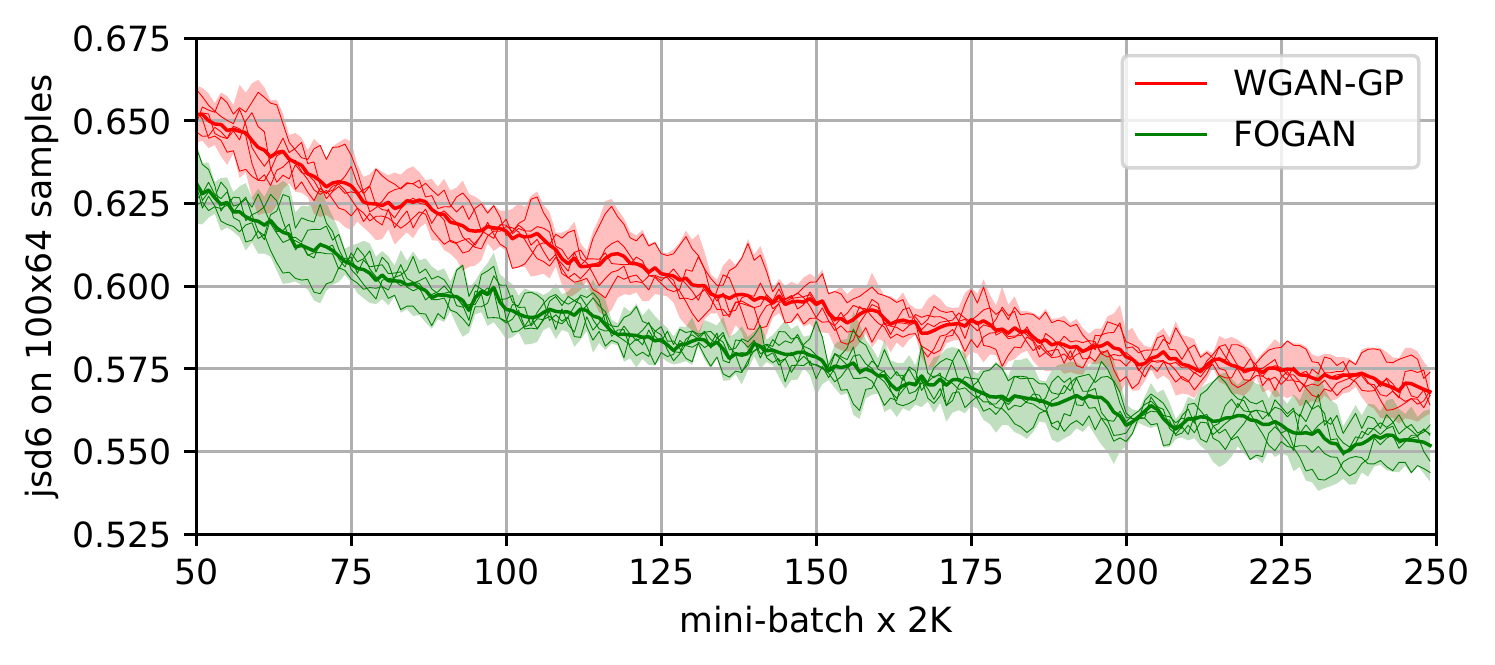}
\end{minipage}
\begin{minipage}{\linewidth}
\centering
\includegraphics[width=\linewidth]{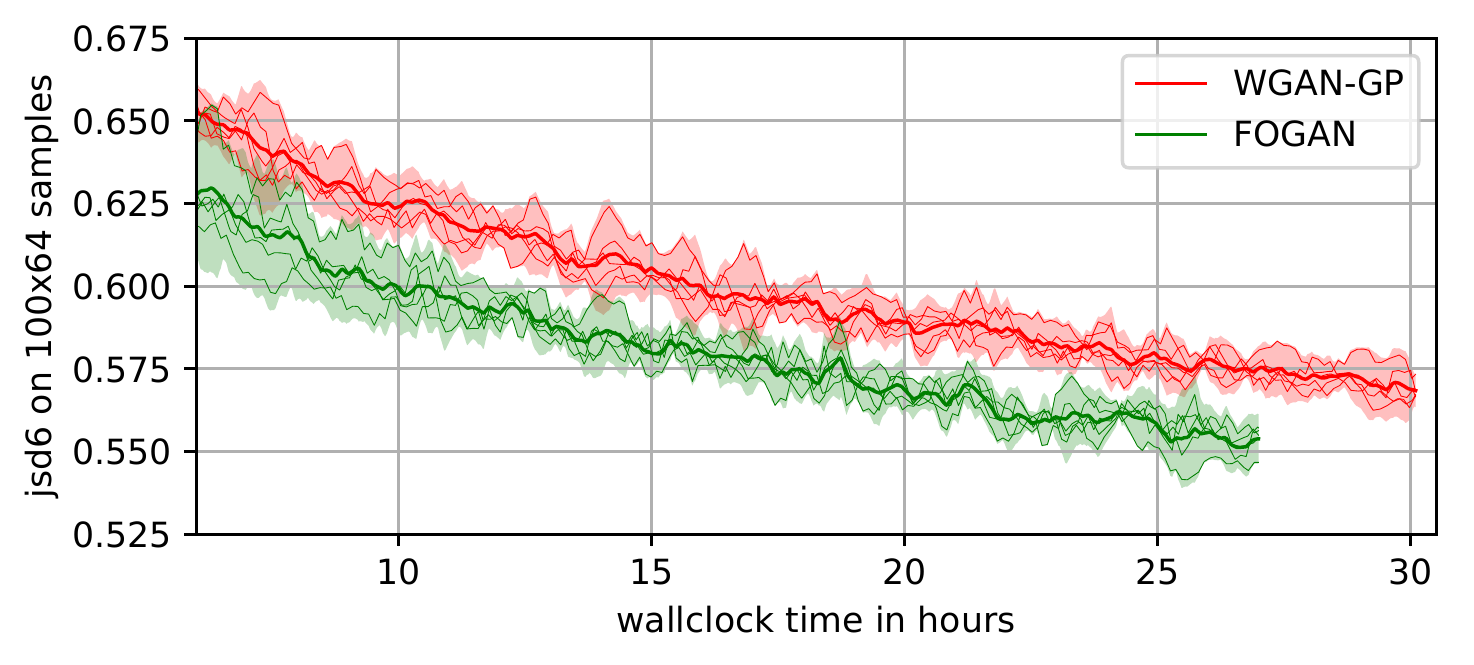}
\end{minipage}
\caption{
Five training runs of both WGAN-GP and FOGAN, with the average of all runs plotted in bold and
the $2\sigma$ error margins denoted by shaded regions.
For easy visualization, we plot the moving average of the last three $n$-gram JSD estimations.
The first two plots both show training w.r.t.\ number of training iterations; the second plot starts at iteration 50.
The last plot show training with respect to wall-clock time, starting after 6 hours of training.
}\label{F:jsd_6}
\end{figure}
\clearpage

 \section*{Acknowledgements}
 This work was supported by Zalando SE with Research Agreement 01/2016.
 
 \bibliography{bibi}

\begin{thebibliography}{35}
\providecommand{\natexlab}[1]{#1}
\providecommand{\url}[1]{\texttt{#1}}
\expandafter\ifx\csname urlstyle\endcsname\relax
  \providecommand{\doi}[1]{doi: #1}\else
  \providecommand{\doi}{doi: \begingroup \urlstyle{rm}\Url}\fi

\bibitem[Arjovsky \& Bottou(2017)Arjovsky and Bottou]{arjovsky2017towards}
Arjovsky, M. and Bottou, L.
\newblock Towards principled methods for training generative adversarial
  networks.
\newblock In \emph{International Conference of Learning Representations
  (ICLR)}, 2017.

\bibitem[Arjovsky et~al.(2017)Arjovsky, Chintala, and
  Bottou]{arjovsky2017wasserstein}
Arjovsky, M., Chintala, S., and Bottou, L.
\newblock {{W}asserstein Generative Adversarial Networks}.
\newblock In \emph{Proceedings of the 34th International Conference on Machine
  Learning (ICML)}, 2017.

\bibitem[Bellemare et~al.(2017)Bellemare, Danihelka, Dabney, Mohamed,
  Lakshminarayanan, Hoyer, and Munos]{bellemare2017cramer}
Bellemare, M.~G., Danihelka, I., Dabney, W., Mohamed, S., Lakshminarayanan, B.,
  Hoyer, S., and Munos, R.
\newblock The cramer distance as a solution to biased wasserstein gradients.
\newblock \emph{arXiv preprint arXiv:1705.10743}, 2017.

\bibitem[Bergmann et~al.(2017)Bergmann, Jetchev, and Vollgraf]{BergmannJV17}
Bergmann, U., Jetchev, N., and Vollgraf, R.
\newblock Learning texture manifolds with the periodic spatial {GAN}.
\newblock In \emph{Proceedings of the 34th International Conference on Machine
  Learning, (ICML)}, 2017.

\bibitem[Bishop et~al.(1998)Bishop, Svens{\'e}n, and Williams]{bishop1998gtm}
Bishop, C.~M., Svens{\'e}n, M., and Williams, C.~K.
\newblock {GTM}: The generative topographic mapping.
\newblock \emph{Neural computation}, 10\penalty0 (1):\penalty0 215--234, 1998.

\bibitem[Bińkowski et~al.(2018)Bińkowski, Sutherland, Arbel, and
  Gretton]{binkowski2018demystifying}
Bińkowski, M., Sutherland, D.~J., Arbel, M., and Gretton, A.
\newblock Demystifying {MMD} {GAN}s.
\newblock In \emph{International Conference on Learning Representations
  (ICLR)}, 2018.

\bibitem[Chelba et~al.(2013)Chelba, Mikolov, Schuster, Ge, Brants, Koehn, and
  Robinson]{chelba2013one}
Chelba, C., Mikolov, T., Schuster, M., Ge, Q., Brants, T., Koehn, P., and
  Robinson, T.
\newblock One billion word benchmark for measuring progress in statistical
  language modeling.
\newblock \emph{arXiv preprint arXiv:1312.3005}, 2013.

\bibitem[Dziugaite et~al.(2015)Dziugaite, Roy, and
  Ghahramani]{dziugaite2015training}
Dziugaite, G.~K., Roy, D.~M., and Ghahramani, Z.
\newblock Training generative neural networks via maximum mean discrepancy
  optimization.
\newblock In \emph{Proceedings of the 31st Conference on Uncertainty in
  Artificial Intelligence (UAI)}, 2015.

\bibitem[Fedus et~al.(2017)Fedus, Rosca, Lakshminarayanan, Dai, Mohamed, and
  Goodfellow]{fedus2017many}
Fedus, W., Rosca, M., Lakshminarayanan, B., Dai, A.~M., Mohamed, S., and
  Goodfellow, I.
\newblock Many paths to equilibrium: {GAN}s do not need to decrease a
  divergence at every step.
\newblock \emph{arXiv preprint arXiv:1710.08446}, 2017.

\bibitem[Goodfellow et~al.(2014)Goodfellow, Pouget-Abadie, Mirza, Xu,
  Warde-Farley, Ozair, Courville, and Bengio]{goodfellow2014generative}
Goodfellow, I., Pouget-Abadie, J., Mirza, M., Xu, B., Warde-Farley, D., Ozair,
  S., Courville, A., and Bengio, Y.
\newblock Generative adversarial nets.
\newblock In \emph{Advances in Neural Information Processing Systems 27
  (NIPS)}, 2014.

\bibitem[Gulrajani et~al.(2017)Gulrajani, Ahmed, Arjovsky, Dumoulin, and
  Courville]{gulrajani2017improved}
Gulrajani, I., Ahmed, F., Arjovsky, M., Dumoulin, V., and Courville, A.~C.
\newblock {Improved Training of Wasserstein {GAN}s}.
\newblock In \emph{Advances in Neural Information Processing Systems 30
  (NIPS)}, 2017.

\bibitem[Heusel et~al.(2017)Heusel, Ramsauer, Unterthiner, Nessler, and
  Hochreiter]{heusel2017gans}
Heusel, M., Ramsauer, H., Unterthiner, T., Nessler, B., and Hochreiter, S.
\newblock {GAN}s trained by a two time-scale update rule converge to a local
  nash equilibrium.
\newblock In \emph{Advances in Neural Information Processing Systems 30
  (NIPS)}, 2017.

\bibitem[Jetchev et~al.(2016)Jetchev, Bergmann, and
  Vollgraf]{jetchev2016texture}
Jetchev, N., Bergmann, U., and Vollgraf, R.
\newblock Texture synthesis with spatial generative adversarial networks.
\newblock \emph{arXiv preprint arXiv:1611.08207}, 2016.

\bibitem[Jetchev et~al.(2017)Jetchev, Bergmann, and
  Seward]{jetchev2017ganosaic}
Jetchev, N., Bergmann, U., and Seward, C.
\newblock {GAN}osaic: Mosaic creation with generative texture manifolds.
\newblock \emph{arXiv preprint arXiv:1712.00269}, 2017.

\bibitem[Kodali et~al.(2017)Kodali, Abernethy, Hays, and Kira]{dragan}
Kodali, N., Abernethy, J., Hays, J., and Kira, Z.
\newblock How to train your {DRAGAN}.
\newblock \emph{arXiv preprint arXiv:1705.07215}, 2017.

\bibitem[Krizhevsky \& Hinton(2009)Krizhevsky and
  Hinton]{krizhevsky2009learning}
Krizhevsky, A. and Hinton, G.
\newblock Learning multiple layers of features from tiny images.
\newblock 2009.

\bibitem[Ledig et~al.(2017)Ledig, Theis, Husz{\'a}r, Caballero, Cunningham,
  Acosta, Alejandro, Aitken, Tejani, Totz, Wang, and Shi]{ledig2016photo}
Ledig, C., Theis, L., Husz{\'a}r, F., Caballero, J., Cunningham, A., Acosta,
  Alejandro, Aitken, A., Tejani, A., Totz, J., Wang, Z., and Shi, W.
\newblock Photo-realistic single image super-resolution using a generative
  adversarial network.
\newblock In \emph{2017 IEEE Conference on Computer Vision and Pattern
  Recognition (CVPR)}, pp.\  105--114, July 2017.
\newblock \doi{10.1109/CVPR.2017.19}.

\bibitem[Li et~al.(2015)Li, Swersky, and Zemel]{li2015generative}
Li, Y., Swersky, K., and Zemel, R.
\newblock Generative moment matching networks.
\newblock In \emph{Proceedings of the 32nd International Conference on Machine
  Learning (ICML)}, 2015.

\bibitem[Liu et~al.(2017)Liu, Bousquet, and Chaudhuri]{liu2017approximation}
Liu, S., Bousquet, O., and Chaudhuri, K.
\newblock Approximation and convergence properties of generative adversarial
  learning.
\newblock In \emph{Advances in Neural Information Processing Systems 30
  (NIPS)}, pp.\  5545--5553, 2017.

\bibitem[Liu et~al.(2015)Liu, Luo, Wang, and Tang]{liu2015deep}
Liu, Z., Luo, P., Wang, X., and Tang, X.
\newblock Deep learning face attributes in the wild.
\newblock In \emph{Proceedings of the IEEE International Conference on Computer
  Vision}, pp.\  3730--3738, 2015.

\bibitem[Mescheder et~al.(2017)Mescheder, Nowozin, and
  Geiger]{mescheder2017numerics}
Mescheder, L., Nowozin, S., and Geiger, A.
\newblock The numerics of {GAN}s.
\newblock In \emph{Advances in Neural Information Processing Systems 30
  (NIPS)}, 2017.

\bibitem[Mescheder et~al.(2018)Mescheder, Geiger, and
  Nowozin]{mescheder2018which}
Mescheder, L., Geiger, A., and Nowozin, S.
\newblock Which training methods for gans do actually converge.
\newblock \emph{arXiv preprint arXiv:1801.04406v2}, 2018.

\bibitem[Metz et~al.(2017)Metz, Poole, Pfau, and
  Sohl-Dickstein]{metz2016unrolled}
Metz, L., Poole, B., Pfau, D., and Sohl-Dickstein, J.
\newblock Unrolled generative adversarial networks.
\newblock In \emph{International Conference of Learning Representations
  (ICLR)}, 2017.

\bibitem[Milgrom \& Segal(2002)Milgrom and Segal]{milgrom2002envelope}
Milgrom, P. and Segal, I.
\newblock Envelope theorems for arbitrary choice sets.
\newblock \emph{Econometrica}, 70\penalty0 (2):\penalty0 583--601, 2002.

\bibitem[Mroueh et~al.(2018)Mroueh, Li, Sercu, Raj, and
  Cheng]{mroueh2017sobolev}
Mroueh, Y., Li, C.-L., Sercu, T., Raj, A., and Cheng, Y.
\newblock Sobolev {GAN}.
\newblock In \emph{International Conference on Learning Representations
  (ICLR)}, 2018.

\bibitem[Nagarajan \& Kolter(2017)Nagarajan and Kolter]{nagarajan2017gradient}
Nagarajan, V. and Kolter, J.~Z.
\newblock Gradient descent {GAN} optimization is locally stable.
\newblock In \emph{Advances in Neural Information Processing Systems 30
  (NIPS)}, 2017.

\bibitem[Nowozin et~al.(2016)Nowozin, Cseke, and Tomioka]{nowozin2016f}
Nowozin, S., Cseke, B., and Tomioka, R.
\newblock {F-GAN}: Training generative neural samplers using variational
  divergence minimization.
\newblock In \emph{Advances in Neural Information Processing Systems 29
  (NIPS)}, 2016.

\bibitem[Petzka et~al.(2018)Petzka, Fischer, and
  Lukovnicov]{petzka2017regularization}
Petzka, H., Fischer, A., and Lukovnicov, D.
\newblock On the regularization of wasserstein {GAN}s.
\newblock In \emph{International Conference on Learning Representations
  (ICLR)}, 2018.

\bibitem[Radford et~al.(2016)Radford, Metz, and
  Chintala]{radford2015unsupervised}
Radford, A., Metz, L., and Chintala, S.
\newblock Unsupervised representation learning with deep convolutional
  generative adversarial networks.
\newblock In \emph{International Conference of Learning Representations
  (ICLR)}, 2016.

\bibitem[Salimans et~al.(2016)Salimans, Goodfellow, Zaremba, Cheung, Radford,
  and Chen]{salimans2016improved}
Salimans, T., Goodfellow, I., Zaremba, W., Cheung, V., Radford, A., and Chen,
  X.
\newblock Improved techniques for training {GAN}s.
\newblock In \emph{Advances in Neural Information Processing Systems 29
  (NIPS)}, pp.\  2234--2242, 2016.

\bibitem[Schmidhuber(1992)]{schmidhuber1992learning}
Schmidhuber, J.
\newblock Learning factorial codes by predictability minimization.
\newblock \emph{Neural Computation}, 4\penalty0 (6):\penalty0 863--879, 1992.

\bibitem[Sriperumbudur et~al.(2010)Sriperumbudur, Gretton, Fukumizu,
  Sch{\"o}lkopf, and Lanckriet]{sriperumbudur2010hilbert}
Sriperumbudur, B.~K., Gretton, A., Fukumizu, K., Sch{\"o}lkopf, B., and
  Lanckriet, G.~R.
\newblock Hilbert space embeddings and metrics on probability measures.
\newblock \emph{Journal of Machine Learning Research}, 11\penalty0
  (Apr):\penalty0 1517--1561, 2010.

\bibitem[Srivastava et~al.(2017)Srivastava, Valkov, Russell, Gutmann, and
  Sutton]{srivastava2017veegan}
Srivastava, A., Valkov, L., Russell, C., Gutmann, M., and Sutton, C.
\newblock {VEEGAN}: Reducing mode collapse in {GAN}s using implicit variational
  learning.
\newblock In \emph{Advances in Neural Information Processing Systems 30
  (NIPS)}, 2017.

\bibitem[Unterthiner et~al.(2018)Unterthiner, Nessler, Seward, Klambauer,
  Heusel, Ramsauer, and Hochreiter]{unterthiner2017coulomb}
Unterthiner, T., Nessler, B., Seward, C., Klambauer, G., Heusel, M., Ramsauer,
  H., and Hochreiter, S.
\newblock Coulomb {GAN}s: Provably optimal nash equilibria via potential
  fields.
\newblock \emph{International Conference of Learning Representations (ICLR)},
  2018.

\bibitem[Yu et~al.(2015)Yu, Zhang, Song, Seff, and Xiao]{yu15lsun}
Yu, F., Zhang, Y., Song, S., Seff, A., and Xiao, J.
\newblock Lsun: Construction of a large-scale image dataset using deep learning
  with humans in the loop.
\newblock \emph{arXiv preprint arXiv:1506.03365}, 2015.

\end{thebibliography}
 \bibliographystyle{icml2018_style/icml2018}
 \appendix
 \onecolumn
 \section{Proof of Things}\label{S:proof_of_things}

\begin{proof}[Proof of Theorem \ref{T:pwgan_theorem}]
 The proof of this theorem is split into smaller lemmas that are proven individually.
 \begin{itemize}
  \item That $\tau_P$ is a strict adversarial divergence which is equivalent to $\tau_W$ is proven in Lemma \ref{L:equavent},
  thus showing that $\tau_P$ fulfills Requirement \ref{R:nonzero}.
  \item $\tau_P$ fulfills Requirement \ref{R:convex_admissible} by design.
  \item The existence of an optimal critic in $\oc_{\tau_P}(\mathbb P,\mathbb Q)$ follows directly from Lemma \ref{L:existence}.
  \item That there exists a critic $f^*\in\oc_{\tau_P}(\mathbb P,\mathbb Q)$ that fulfills Eq.\ \ref{E:first_order_assumption} is because
  Lemma \ref{L:existence} ensures that a continuous differentiable $f^*$ exists in $\oc_{\tau_P}(\mathbb P,\mathbb Q)$ which
  fulfills Eq.\ \ref{E:slope_strict}. Because Eq.\ \ref{E:slope_strict} holds for $f^*\in C(X)$,
  the same reasoning as the end of the proof of Lemma \ref{L:first_order} can be used to show Requirement \ref{R:fo_divergence}
 \end{itemize}
\end{proof}

We prepare by showing a few basic lemmas used in the remaining proofs

\begin{lemma}[concavity of $\tau_P(\mathbb P\Vert\mathbb Q;\cdot)$]\label{L:pen_concave}
 The mapping $C^1(X)\to\mathbb R$, $f\mapsto\tau_P(\mathbb P\Vert\mathbb Q;f)$ is concave.
 \end{lemma}

 \begin{proof}
  The concavity of $f\mapsto\mathbb E_{x\sim\mathbb P}[f(x)]-\mathbb E_{x'\sim\mathbb Q}[f(x')]$ is trivial. Now consider $\gamma\in(0,1)$, then
  \begin{align*}
   &\mathbb E_{x\sim\mathbb P,x'\sim\mathbb Q}\left[\frac{(\gamma (f(x)-f(x')) + (1-\gamma)(\hat f(x)-\hat f(x')))^2}{\Vert x-x'\Vert}\right] \\
   \leq &\mathbb E_{x\sim\mathbb P,x'\sim\mathbb Q}\left[\frac{\gamma (f(x)-f(x'))^2 + (1-\gamma)(\hat f(x)-\hat f(x'))^2}{\Vert x-x'\Vert}\right] \\
   = &\gamma\mathbb E_{x\sim\mathbb P,x'\sim\mathbb Q}\left[\frac{(f(x)-f(x'))^2}{\Vert x-x'\Vert}\right]
   +(1-\gamma)\mathbb E_{x\sim\mathbb P,x'\sim\mathbb Q}\left[\frac{(\hat f(x)-\hat f(x'))^2}{\Vert x-x'\Vert}\right],
  \end{align*}
  thus showing concavity of $\tau_P(\mathbb P\Vert\mathbb Q;\cdot)$.
 \end{proof}

\begin{lemma}[necessary and sufficient condition for maximum]\label{L:necessary_for_maximum}
  Assume $\mathbb P,\mathbb Q\in\mathcal P(X)$ fulfill assumptions \ref{A:1} and \ref{A:2}. Then for any $f\in \oc_{\tau_P}(\mathbb P,\mathbb Q)$ it must hold that
  \begin{equation}\label{E:slope}
  P_{x'\sim\mathbb Q}\left(\mathbb E_{x\sim\mathbb P}\left[\frac{f(x)-f(x')}{\Vert x-x'\Vert}\right] = \frac{1}{2\lambda}\right)=1
  \end{equation}
  and
  \begin{equation}\label{E:slope_p}
  P_{x\sim\mathbb P}\left(\mathbb E_{x'\sim\mathbb Q}\left[\frac{f(x)-f(x')}{\Vert x-x'\Vert}\right] = \frac{1}{2\lambda}\right)=1.
  \end{equation}
  Further, if $f\in C^1(X)$ and fulfills Eq.\ \ref{E:slope} and \ref{E:slope_p}, then $f\in \oc_{\tau_P}(\mathbb P,\mathbb Q)$
 \end{lemma}

 \begin{proof}
  Since in Lemma \ref{L:pen_concave} it was shown that the the mapping $f\mapsto\tau_P(\mathbb P\Vert\mathbb Q,f)$ is concave,
  $f\in\oc_\tau(\mathbb P,\mathbb Q)$ if and only if  $f\in C^1(X)$ and $f$ is a local maximum of $\tau_P(\mathbb P\Vert\mathbb Q;\cdot)$.
  This is equivalent to saying that all $u_1,u_2\in C^1(X)$ with $\supp(u_1)\cap\supp(\mathbb Q)=\emptyset$ and $\supp(u_2)\cap\supp(\mathbb P)=\emptyset$ it holds
  \[\nabla_{(\varepsilon,\rho)}\left[\mathbb E_{\mathbb P}[f+\varepsilon u_1]-\mathbb E_{\mathbb Q}[f+\rho u_2]
  -\lambda\mathbb E_{x\sim\mathbb P,x'\sim\mathbb Q}\left[\frac{((f+\varepsilon u_1)(x)-(f+\rho u_2)(x'))^2}{\Vert x-x'\Vert}\right]\right]
  \bigg|_{\varepsilon=0,\rho=0}=0\]
  which holds if and only if
  \[\mathbb E_{x\sim\mathbb P}\left[u_1(x)\left(1-2\lambda\mathbb E_{x'\sim\mathbb Q}\left[\frac{(f(x)-f(x'))}{\Vert x-x'\Vert}\right]\right)\right]=0\]
  and
  \[\mathbb E_{x'\sim\mathbb Q}\left[u_2(x')\left(1-2\lambda\mathbb E_{x\sim\mathbb P}\left[\frac{(f(x)-f(x'))}{\Vert x-x'\Vert}\right]\right)\right]=0\]
  proving that Eq.\ \ref{E:slope} and \ref{E:slope_p} are necessary and sufficient.
 \end{proof}

 \begin{lemma}\label{L:existence}
  Let $\mathbb P,\mathbb Q\in\mathcal P(X)$ be probability measures fulfilling Assumptions \ref{A:1} and \ref{A:2}.
  Define an open subset of $X$, $\Omega\subseteq X$, such that $\supp(\mathbb Q)\subseteq\Omega$ and $\inf_{x\in\supp(\mathbb P),x'\in\Omega}\Vert x-x'\Vert>0$.
  Then there exists a $f\in\mathcal F=C^1(X)$ such that
  \begin{equation}\label{E:slope_strict}
   \forall x'\in\Omega:\quad\mathbb E_{x\sim\mathbb P}\left[\frac{f(x)-f(x')}{\Vert x-x'\Vert}\right]=\frac {1}{2\lambda}
  \end{equation}
  and
  \begin{equation}\label{E:slope_strict_p}
   \forall x\in\supp(\mathbb P):\quad\mathbb E_{x'\sim\mathbb Q}\left[\frac{f(x)-f(x')}{\Vert x-x'\Vert}\right]=\frac {1}{2\lambda}
  \end{equation}
  and $\tau_P(\mathbb P\Vert\mathbb Q;f)=\tau_P(\mathbb P\Vert\mathbb Q)$.
 \end{lemma}

 \begin{proof}
  Since $\tau(\mathbb P\Vert\mathbb Q;f)=\tau(\mathbb P\Vert\mathbb Q;f+c)$ for any $c\in\mathbb R$ and is only affected by
  values of $f$ on $\supp(\mathbb P)\cup\Omega$ we first start by considering
  \[\mathcal F=\left\{f\in C^1(\supp(\mathbb P)\cup\Omega)\mid \mathbb E_{x\sim\mathbb P,x'\sim\mathbb Q}\left[\frac{f(x')}{\Vert x-x'\Vert}\right]=0\right\}\]
  Observe that Eq.\ \ref{E:slope_strict} holds if
 \[x'\in\Omega:\quad f(x')=\frac{\mathbb E_{x\sim\mathbb P}[\frac{f(x)}{\Vert x-x'\Vert}]-\frac{1}{2\lambda}}{\mathbb E_{x\sim\mathbb P}[\frac{1}{\Vert x-x'\Vert}]}\]
 and similarly for Eq.\ \ref{E:slope_strict_p}
 \[\forall x\in\supp(\mathbb P):\quad f(x)=\frac{\mathbb E_{x'\sim\mathbb Q}[\frac{f(x')}{\Vert x-x'\Vert}]+\frac{1}{2\lambda}}{\mathbb E_{x'\sim\mathbb Q}[\frac{1}{\Vert x-x'\Vert}]}.\]
 Now it's clear that if the mapping $T:\mathcal F\to\mathcal F$ defined by
 \begin{equation}\label{E:T_def}
   T(f)(x):=\begin{cases}
   \frac{\mathbb E_{x'\sim\mathbb Q}[\frac{f(x')}{\Vert x-x'\Vert}]+\frac{1}{2\lambda}}{\mathbb E_{x'\sim\mathbb Q}[\frac{1}{\Vert x-x'\Vert}]} & x\in\supp(\mathbb P) \\
   \frac{\mathbb E_{x'\sim\mathbb P}[\frac{f(x')}{\Vert x-x'\Vert}]-\frac{1}{2\lambda}}{\mathbb E_{x'\sim\mathbb P}[\frac{1}{\Vert x-x'\Vert}]} & x\in\Omega
  \end{cases}
 \end{equation}
  admit a fix point $f^*\in\mathcal F$, i.e.\ $T(f^*)=f^*$, then $f^*$ is a solution to Eq.\ \ref{E:slope_strict} and \ref{E:slope_strict_p},
  and with that a solution to Eq.\ \ref{E:slope} and \ref{E:slope_p} and $\tau_P(\mathbb P\Vert\mathbb Q;f^*)=\tau_P(\mathbb P\Vert\mathbb Q)$.

  Define the mapping $S:\mathcal F\to\mathcal F$ by
  \[S(f)(x)=\frac{f(x)}{2\lambda\mathbb E_{\tilde x\sim\mathbb P,x'\sim\mathbb Q}\left[\frac{f(\tilde x)-f(x')}{\Vert\tilde x-x'\Vert}\right]}.\]
  Then
  \begin{equation}\label{E:s_equality}
  \mathbb E_{\tilde x\sim\mathbb P,x'\sim\mathbb Q}\left[\frac{S(f)(\tilde x)-S(f)(x')}{\Vert\tilde x-x'\Vert}\right]=\frac{1}{2\lambda}
  \end{equation}
  and
  \[S(S(f))(x)=\frac{S(f)(x)}{2\lambda\mathbb E_{\tilde x\sim\mathbb P,x'\sim\mathbb Q}\left[\frac{S(f)(\tilde x)-S(f)(x')}{\Vert\tilde x-x'\Vert}\right]}=\frac{S(f)(x)}{2\lambda\frac{1}{2\lambda}}=S(f)(x)\]
  making $S$ a projection.
  By the same reasoning, if $\mathbb E_{\tilde x\sim\mathbb P,x'\sim\mathbb Q}\left[\frac{f(\tilde x)-f(x')}{\Vert\tilde x-x'\Vert}\right]=\frac{1}{2\lambda}$
  then $f$ is a fix-point of $S$, i.e.\ $S(f)=f$.
  Assume $f$ is such a function, then by definition  of $T$ in Eq.\ \ref{E:T_def}
  \begin{align*}\mathbb E_{\tilde x\sim\mathbb P,x'\sim\mathbb Q}\left[\frac{T(f)(\tilde x)-T(f)(x')}{\Vert\tilde x-x'\Vert}\right]
   &=\mathbb E_{\tilde x\sim\mathbb P}\left[\mathbb E_{x'\sim\mathbb Q}\left[\frac{T(f)(\tilde x)}{\Vert\tilde x-x'\Vert}\right]\right]
   -\mathbb E_{x'\sim\mathbb Q}\left[\mathbb E_{\tilde x\sim\mathbb P}\left[\frac{T(f)(x')}{\Vert\tilde x-x'\Vert}\right]\right]\\
   &=\mathbb E_{\tilde x\sim\mathbb P}\left[\mathbb E_{x'\sim\mathbb Q}\left[\frac{f(x')}{\Vert\tilde x-x'\Vert}\right]+\frac{1}{2\lambda}\right]
   -\mathbb E_{x'\sim\mathbb Q}\left[\mathbb E_{\tilde x\sim\mathbb P}\left[\frac{f(\tilde x)}{\Vert\tilde x-x'\Vert}\right]-\frac{1}{2\lambda}\right]\\
   &=-\mathbb E_{\tilde x\sim\mathbb P,x'\sim\mathbb Q}\left[\frac{f(\tilde x)-f(x')}{\Vert\tilde x-x'\Vert}\right]+2\frac{1}{2\lambda}\\
   &=\frac{1}{2\lambda}.
  \end{align*}
  Therefore, $S(T(S(f)))=T(S(f))$. We can define $S(\mathcal F)=\{S(f)\mid f\in\mathcal F\}$ and see that $T:S(\mathcal F)\to S(\mathcal F)$.
  Further, since $S(\cdot)$ only multiplies with a scalar, $S(F)\subseteq\mathcal F$.

  Let $f_1,f_2\in S(\mathcal F)$. From Eq.\ \ref{E:s_equality} we get
  \[
  \mathbb E_{x\sim\mathbb P,x'\sim\mathbb Q}\left[\frac{f_1(x')-f_2(x')}{\Vert x-x'\Vert}\right]
   =\mathbb E_{x\sim\mathbb P,x'\sim\mathbb Q}\left[\frac{f_1(x)-f_2(x)}{\Vert x-x'\Vert}\right].
  \]
  Now since for every $f\in\mathcal F$ it holds by design that $\mathbb  E_{x\sim\mathbb P,x'\sim\mathbb Q}\left[\frac{f(x')}{\Vert x-x'\Vert}\right]=0$
  and since $S(\mathcal F)\subseteq\mathcal F$ we see that $f_1,f_2\in S(\mathcal F)$ that
  \begin{equation*}
  \mathbb E_{x\sim\mathbb P,x'\sim\mathbb Q}\left[\frac{f_1(x')-f_2(x')}{\Vert x-x'\Vert}\right]
   =\mathbb E_{x\sim\mathbb P,x'\sim\mathbb Q}\left[\frac{f_1(x)-f_2(x)}{\Vert x-x'\Vert}\right]=0
  \end{equation*}
  Using this with the continuity of $f_1,f_2$, there must exist $x_1\in\supp(\mathbb P)$ with
  \begin{equation*}
  \mathbb E_{x'\sim\mathbb Q}\left[\frac{f_1(x')-f_2(x')}{\Vert x_1-x'\Vert}\right]=0
  \end{equation*}
  With this (and compactness of our domain),
  $\mathbb Q$ must have mass in both positive and negative regions of $f_1-f_2$ and exists a constant $p<1$ such that for all
  $f_1,f_2\in S(\mathcal F)$ it holds
  \begin{equation}\label{E:banach_secret_sauce}
   \sup_{x\in\supp(\mathbb P)}\left|\mathbb E_{x'\sim\mathbb Q}\left[\frac{f_1(x')-f_2(x')}{\Vert x-x'\Vert}\right]\right|
   \leq p\sup_{x\in\supp(\mathbb P)}\mathbb E_{x'\sim\mathbb Q}\left[\frac{1}{\Vert x-x'\Vert}\right]\sup_{x'\in\Omega} |f_1(x')-f_2(x')|.
  \end{equation}

  To show the existence of a fix-point for $T$ in the Banach Space $(\mathcal F,\Vert\cdot\Vert_{\infty})$ we use the Banach fixed-point theorem
  to show that $T$ has a fixed point in the metric space $(S(\mathcal F),\Vert\cdot\Vert_{\infty})$ (remember that $T:S(\mathcal F)\to S(\mathcal F)$ and $S(\mathcal F)\subseteq \mathcal F$).
  If $f_1,f_2\in S(\mathcal F)$ then
  \begin{align*}
  \sup_{x\in\supp(\mathbb P)} |T(f_1)(x)-T(f_2)(x)|
  &=\sup_{x\in\supp(\mathbb P)}\left\vert
  \frac{
  \mathbb E_{x'\sim\mathbb Q}\left[\frac{f_1(x')-f_2(x')}{\Vert x-x'\Vert}\right]
  }{
  \mathbb E_{x'\sim\mathbb Q}\left[\frac{1}{\Vert x-x'\Vert}\right]
  }\right\vert\\
  &\leq p\sup_{x'\in\supp(\mathbb Q)}|f_1(x')-f_2(x')|\quad\text{using Eq.\ \ref{E:banach_secret_sauce}}
  \end{align*}
 The same trick can be used to find some some $q<1$ and show
 \[\sup_{x'\in\Omega} |T(f_1)(x')-T(f_2)(x')|\leq q\sup_{x\in\supp(\mathbb P)}|f_1(x)-f_2(x)|\]
 thereby showing
 \[\Vert T(f_1)-T(f_2) \Vert_\infty<\max(p,q) \Vert f_1-f_2 \Vert_\infty\]
 The Banach fix-point theorem then delivers the existence of a fix-point $f^*\in S(\mathcal F)$ for $T$.

 Finally, we can use the Tietze extension theorem to extend $f^*$
 to all of $X$, thus finding a fix point for $T$ in $C^1(X)$ and proving the lemma.
 \end{proof}

 \begin{lemma}\label{L:equavent}
  $\tau_P$ is a strict adversarial divergence and $\tau_P$ and $\tau_W$ are equivalent.
 \end{lemma}

 \begin{proof}
  Let $\mathbb P,\mathbb Q\in\mathcal P(X)$ be two probability measures fulfilling Assumptions \ref{A:1} and \ref{A:2} with $\mathbb P\neq\mathbb Q$.
  It's shown in \cite{sriperumbudur2010hilbert} that $\mu=\tau_W(\mathbb P,\mathbb Q)>0$,
  meaning there exists a function $f\in C(X)$, $\Vert f\Vert_L \leq 1$ such that
 \[\mathbb E_{\mathbb P}[f]-\mathbb E_{\mathbb Q}[f]=\mu>0.\]
 The Stone–Weierstrass theorem tells us that there exists a $f'\in C_\infty(X)$ such that $\Vert f-f'\Vert_\infty\leq\frac\mu 4$
 and thus $\mathbb E_{\mathbb P}[f']-\mathbb E_{\mathbb Q}[f']\geq\frac{\mu}{2}$.
 Now consider the function $\varepsilon f'$ with $\varepsilon>0$, it's clear that
 \[\tau_P(\mathbb P\Vert\mathbb Q)\geq\tau_P(\mathbb P\Vert\mathbb Q;\varepsilon f')
 =\varepsilon(\underbrace{\mathbb E_{\mathbb P}[f']-\mathbb E_{\mathbb Q}[f']}_{\geq\frac{\mu}{2}})
 -\varepsilon^2\lambda\mathbb E_{x\sim \mathbb P,x'\sim\mathbb Q}[\frac{(f'(x)-f'(x'))^2}{\Vert x-x'\Vert}]\]
 and so for a sufficiently small $\varepsilon>0$ we'll get $\tau_P(\mathbb P\Vert\mathbb Q;\varepsilon f')>0$
 meaning $\tau_P(\mathbb P\Vert\mathbb Q)>0$ and $\tau_P$ is a strict adversarial divergence.

 To show equivalence, we note that
 \[\tau_P(\mathbb P\Vert\mathbb Q)\leq\sup_{m\in C(X^2)}\mathbb E_{x\sim\mathbb P,x'\sim\mathbb Q}\left[m(x,x')\left(1-\lambda\frac{m(x,x')}{\Vert x-x'\Vert}\right)\right]\]
 therefore for any optimum it must hold $m(x,x')\leq \frac{\Vert x-x'\Vert}{2\lambda}$, and thus (similar to Lemma \ref{L:necessary_for_maximum}) any optimal solution will
 be Lipschitz continuous with a the Lipschitz constant independent of $\mathbb P,\mathbb Q$. Thus $\tau_W(\mathbb P\Vert\mathbb Q)\geq \gamma\tau_P(\mathbb P\Vert\mathbb Q)$ for
 $\gamma>0$, from which we directly get equivalence.
 \end{proof}

 \begin{proof}[Proof of Theorem \ref{T:fogan_theorem}]
 We start by applying Lemma \ref{L:adversary_subset} giving us
 \begin{itemize}
  \item $\oc_{\tau_F}(\mathbb P,\mathbb Q'_{\theta_0})\not=\emptyset$.
  \item For any $\mathbb P,\mathbb Q\in\mathcal P(X)$ fulfilling Assumptions \ref{A:1} and \ref{A:2}, it holds that
  $\tau_F(\mathbb P\Vert\mathbb Q)=\tau_P(\mathbb P\Vert\mathbb Q)$, meaning $\tau_F$ is like $\tau_P$ a strict adversarial divergence
  which is equivalent to $\tau_W$, showing Requirement \ref{R:nonzero}.
  \item $\tau_F$ fulfills Requirement \ref{R:convex_admissible} by design.
  \item Every $f^*\in\oc_{\tau_F}(\mathbb P,\mathbb Q'_{\theta_0})$ is in $\oc_{\tau_P}(\mathbb P,\mathbb Q'_{\theta_0})\subseteq C^1(X)$,
  therefore $f^*$ the gradient $\nabla_\theta\mathbb E_{\mathbb Q_\theta}[f^*]|_{\theta_0}$ exists. Further Lemma \ref{L:first_order}
  shows that the update rule $\nabla_\theta\mathbb E_{\mathbb Q_\theta}[f^*]|_{\theta_0}$ is unique, thus showing Requirement \ref{R:differentiable}.
  \item Lemma \ref{L:first_order} gives us
 every $f^*\in\oc_{\tau_F}(\mathbb P,\mathbb Q'_{\theta_0})$ with the corresponding update rule fulfills Requirement \ref{R:fo_divergence}, thus proving Theorem \ref{T:fogan_theorem}.
 \end{itemize}
 \end{proof}

 Before we can show this theorem, we must prove a few interesting lemmas about $\tau_F$. The following lemma is quite powerful;
 since $\tau_P(\mathbb P\Vert\mathbb Q)=\tau_F(\mathbb P\Vert\mathbb Q)$ and $\oc_{\tau_F}(\mathbb P,\mathbb Q)\subseteq \oc_{\tau_P}(\mathbb P,\mathbb Q)$
 any property that's proven for $\tau_P$ automatically holds for $\tau_F$.

 \begin{lemma}\label{L:adversary_subset}
If let $X\subseteq\mathbb R^n$ and $\mathbb P,\mathbb Q\in \mathcal P(X)$ be probability measures fulfilling Assumptions \ref{A:1} and \ref{A:2}. Then
\begin{enumerate}
 \item there exists $f^*\in\oc_{\tau_P}(\mathbb P,\mathbb Q)$ so that $\tau_F(\mathbb P\Vert\mathbb Q;f^*)=\tau_P(\mathbb P\Vert\mathbb Q;f^*)$,
 \item $\tau_P(\mathbb P\Vert\mathbb Q)=\tau_F(\mathbb P\Vert\mathbb Q)$,
 \item $\emptyset\not=\oc_{\tau_F}(\mathbb P,\mathbb Q)$,
 \item $\oc_{\tau_F}(\mathbb P,\mathbb Q)\subseteq \oc_{\tau_P}(\mathbb P,\mathbb Q)$.
\end{enumerate}
Clain (4) is especially helpful, now anything that has been proven for all $f^*\in\oc_{\tau_P}(\mathbb P,\mathbb Q)$
automatically holds for all $f^*\in\oc_{\tau_F}(\mathbb P,\mathbb Q)$
\end{lemma}

\begin{proof}
For convenience define
\[G(\mathbb P,\mathbb Q;f):=\mathbb E_{x'\sim\mathbb Q}\left[\left(\Vert\nabla_x f(x)\big|_{x'}\Vert-
  \frac{\left\Vert\mathbb E_{\tilde x\sim\mathbb P}[(\tilde x- x')\frac{f(\tilde x)-f(x')}{\Vert x'-\tilde x\Vert^3}]\right\Vert}
 {\mathbb E_{\tilde x\sim\mathbb P}[\frac{1}{\Vert x'-\tilde x\Vert}]}
  \right)^2\right]\]
($G$ is for gradient penalty) and note that
\[\tau_F(\mathbb P\Vert\mathbb Q;f)=\tau_P(\mathbb P\Vert\mathbb Q;f)-\underbrace{G(\mathbb P,\mathbb Q;f)}_{\geq 0}\]
Therefore it's clear that $\tau_F(\mathbb P\Vert\mathbb Q)\leq\tau_P(\mathbb P\Vert\mathbb Q)$

\noindent\textbf{Claim (1). }
Let $\Omega\subseteq X$ be an open set such that $\supp(\mathbb Q)\subseteq\Omega$ and $\Omega\cap\supp(\mathbb P)=\emptyset$.
Then Lemma \ref{L:existence} tells us there is a $f\in\oc_{\tau_P}(\mathbb P,\mathbb Q)$ (and thus $f\in C^1(X)$) such that
\[\forall x'\in\Omega:\quad\mathbb E_{\tilde x\sim\mathbb P}\left[\frac{f(\tilde x)-f(x')}{\Vert \tilde x-x'\Vert}\right]=\frac{1}{2\lambda}\]
and thus, because $\supp(\mathbb Q)\subseteq\Omega$ open and $f\in C^1(X)$,
\[\forall x'\in\supp(\mathbb Q):\quad\nabla_x\mathbb E_{\tilde x\sim\mathbb P}\left[\frac{f(\tilde x)-f(x)}{\Vert \tilde x-x\Vert}\right]\bigg|_{x'}=0\]
Now taking the gradients with respect to $x'$ gives us
\begin{equation}\label{E:long_equation}
\nabla_x\mathbb E_{\tilde x\sim\mathbb P}\left[\frac{f(\tilde x)-f(x)}{\Vert \tilde x-x\Vert}\right]\bigg|_{x'}
=-\nabla_x f(x)|_{x'}\mathbb E_{\tilde x\sim\mathbb P}\left[\frac{1}{\Vert \tilde x-x'\Vert}\right]
+\mathbb E_{\tilde x\sim\mathbb P}\left[(\tilde x-x')\frac{f(\tilde x)-f(x')}{\Vert \tilde x-x'\Vert^3}\right] 
\end{equation}
meaning
\begin{equation}\label{E:g_zero}
\forall x'\in\supp(\mathbb Q):\quad\nabla_x f(x)|_{x'}=\frac{
\mathbb E_{\tilde x\sim\mathbb P}\left[(\tilde x-x')\frac{f(\tilde x)-f(x')}{\Vert \tilde x-x'\Vert^3}\right]
}{
\mathbb E_{\tilde x\sim\mathbb P}\left[\frac{1}{\Vert \tilde x-x'\Vert}\right]
}
\end{equation}
thus $G(\mathbb P,\mathbb Q;f)=0$, showing the claim.

\noindent\textbf{Claims (2) and (3). }
The claims are a direct result of Claim (1);
for every $\mathbb P,\mathbb Q\in\mathcal P(X)$ there exists a
\[f^*\in\oc_{\tau_P}(\mathbb P,\mathbb Q)\]
such that $G(\mathbb P,\mathbb Q;f^*)=0$. Therefore
\[\tau_P(\mathbb P\Vert\mathbb Q)\geq\tau_F(\mathbb P\Vert\mathbb Q)\geq\tau_F(\mathbb P\Vert\mathbb Q;f^*)=\tau_P(\mathbb P\Vert\mathbb Q;f^*)=\tau_P(\mathbb P\Vert\mathbb Q)\]
thus showing both $\tau_P(\mathbb P\Vert\mathbb Q)=\tau_F(\mathbb P\Vert\mathbb Q)$ and $f^*\in\oc_{\tau_F}(\mathbb P\Vert\mathbb Q)$.

\noindent\textbf{Claim (4). }
This claim is a direct result of claim (2); since $\tau_P(\mathbb P\Vert\mathbb Q)=\tau_F(\mathbb P\Vert\mathbb Q)$,
that means that if $f^*\in\oc_{\tau_F}(\mathbb P\Vert\mathbb Q)$, then
\[\tau_F(\mathbb P\Vert\mathbb Q)=\tau_F(\mathbb P\Vert\mathbb Q;f^*)=\tau_P(\mathbb P\Vert\mathbb Q;f^*)-\underbrace{G(\mathbb P,\mathbb Q;f)}_{\geq 0}\leq \tau_P(\mathbb P\Vert\mathbb Q;f^*)\leq \tau_P(\mathbb P\Vert\mathbb Q)=\tau_F(\mathbb P\Vert\mathbb Q)\]
thus $\tau_P(\mathbb P\Vert\mathbb Q;f^*)= \tau_P(\mathbb P\Vert\mathbb Q)$ and $f^*\in\oc_{\tau_P}(\mathbb P\Vert\mathbb Q)$.
\end{proof}

\begin{lemma}\label{L:first_order_helper}

For every $f^*\in\oc_{\tau_F}(\mathbb P,\mathbb Q')$ it holds

\begin{equation}\label{E:slope_p_last}
 \forall x'\in\supp(\mathbb Q'):\quad \nabla_x\mathbb E_{\tilde x\sim\mathbb P}\left[\frac{f^*(\tilde x)-f^*(x)}{\Vert \tilde x - x\Vert}\right]\bigg|_{x'}=0
\end{equation}

\end{lemma}

\begin{proof}
 Set
\[v=\frac{\mathbb E_{\tilde x\sim\mathbb P}[(\tilde x- x')\frac{f^*(\tilde x)-f^*(x')}{\Vert x'-\tilde x\Vert^3}]}{\left\Vert\mathbb E_{\tilde x\sim\mathbb P}[(\tilde x- x')\frac{f^*(\tilde x)-f^*(x')}{\Vert x'-\tilde x\Vert^3}]\right\Vert}\]
and note that due to construction of $\mathbb Q'$ and $v$, $v$ is such that for almost all $x'\in\supp(\mathbb Q')$ there exists an
$a\neq 0$ where for all $\varepsilon\in[0,|a|]$ it holds $x'+\varepsilon\text{sign}(a) v\in\supp(\mathbb Q')$.

Since $f^*\in C^1(X)$ it holds
\[\frac{d}{d\varepsilon}f^*(x'+\varepsilon v)|_{\varepsilon=0}=\langle v,\nabla_x f^*(x')\rangle.\]
Using Eq.\ \ref{E:slope} we see,
\begin{align*}
&\mathbb E_{x\sim\mathbb P}\left[\frac{f^*(x)-f^*(x'+\varepsilon v)}{\Vert x-(x'+\varepsilon v)\Vert}\right]\\
=&\varepsilon \left\langle v,\nabla_{\hat x}\mathbb E_{x\sim\mathbb P}\left[\frac{f^*(x)-f^*(\hat x)}{\Vert x-\hat x\Vert}\right]\bigg|_{x'}\right\rangle +\mathcal O(\varepsilon^2)\\
=&\frac{\mathbb E_{\tilde x\sim\mathbb P}\left[\left\langle \varepsilon(\tilde x- x'),
\nabla_{\tilde x}\mathbb E_{x\sim\mathbb P}\left[\frac{f^*(x)-f^*(\hat x)}{\Vert x-\hat x\Vert}\right]\Big|_{x'}
\right\rangle\frac{f^*(\tilde x)-f^*(x')}{\Vert x'-\tilde x\Vert^3}\right]}
{\left\Vert\mathbb E_{\tilde x\sim\mathbb P}\left[(\tilde x- x')\frac{f^*(\tilde x)-f^*(x')}{\Vert x'-\tilde x\Vert^3}\right]\right\Vert}
+\mathcal O(\varepsilon^2)\\
=&\frac{\mathbb E_{\tilde x\sim\mathbb P}\left[
\underbrace{\mathbb E_{x\sim\mathbb P}\left[\frac{f^*(x)-f^*(x' + \varepsilon(\tilde x - x'))}{\Vert x-x'+\varepsilon(\tilde x - x')\Vert}\right]}_{=\frac{1}{2\lambda},\text{ Eq. \ref{E:slope} and definition of }\mathbb Q'}
\frac{f^*(\tilde x)-f^*(x')}{\Vert x'-\tilde x\Vert^3}\right]}
{\left\Vert\mathbb E_{\tilde x\sim\mathbb P}\left[(\tilde x- x')\frac{f^*(\tilde x)-f^*(x')}{\Vert x'-\tilde x\Vert^3}\right]\right\Vert}
+\mathcal O(\varepsilon^2)\\
=&\frac{\mathbb E_{\tilde x\sim\mathbb P}\left[
\frac{1}{2\lambda}
\frac{f^*(\tilde x)-f^*(x')}{\Vert x'-\tilde x\Vert^3}\right]}
{\left\Vert\mathbb E_{\tilde x\sim\mathbb P}\left[(\tilde x- x')\frac{f^*(\tilde x)-f^*(x')}{\Vert x'-\tilde x\Vert^3}\right]\right\Vert}
+\mathcal O(\varepsilon^2)
\end{align*}
which means
\begin{align*}
 0&=\frac{d}{d\varepsilon}\mathbb E_{x\sim\mathbb P}\left[\frac{f^*(x)-f^*(x'+\varepsilon v)}{\Vert x-(x'+\varepsilon v)\Vert}\right]\bigg|_{\varepsilon=0}\\
 &=-\frac{d}{d\varepsilon}f^*(x'+\varepsilon v)|_{\varepsilon=0}\mathbb E_{\tilde x\sim\mathbb P}\left[\frac{1}{\Vert \tilde x-x'\Vert}\right]
 -\mathbb E_{x\sim\mathbb P}\left[\langle v, x-x'\rangle\frac{f^*(x)-f^*(x')}{\Vert x-x'\Vert^3}\right].
\end{align*}
Therefore,
\begin{align*}
\frac{d}{d\varepsilon}f^*(x'+\varepsilon v)|_{\varepsilon=0} &= \langle v,\nabla_x f^*(x)|_{x'}\rangle \\
&=\frac{\mathbb E_{\tilde x\sim\mathbb P}\left[\langle v, \tilde x-x'\rangle\frac{f^*(\tilde x)-f^*(x')}{\Vert x-x'\Vert^3}\right]}{\mathbb E_{\tilde x\sim\mathbb P}\left[\frac{1}{\Vert \tilde x-x'\Vert}\right]} \\
&=\frac{\left\langle v,\mathbb E_{\tilde x\sim\mathbb P}[(\tilde x- x')\frac{f^*(\tilde x)-f^*(x')}{\Vert x'-\tilde x\Vert^3}]\right\rangle}{\mathbb E_{\tilde x\sim\mathbb P}\left[\frac{1}{\Vert \tilde x-x'\Vert}\right]} \\
&=\frac{\left\Vert\mathbb E_{\tilde x\sim\mathbb P}[(\tilde x- x')\frac{f^*(\tilde x)-f^*(x')}{\Vert x'-\tilde x\Vert^3}]\right\Vert}{\mathbb E_{\tilde x\sim\mathbb P}\left[\frac{1}{\Vert \tilde x-x'\Vert}\right]}
\end{align*}
Now from the proof of Lemma \ref{L:adversary_subset} claim (4), we know that since $G(\mathbb P,\mathbb Q;f^*)=0$ we get
\[\Vert\nabla_x f^*(x)\big|_{x'}\Vert
 =\frac{\left\Vert\mathbb E_{\tilde x\sim\mathbb P}[(\tilde x- x')\frac{f^*(\tilde x)-f^*(x')}{\Vert x'-\tilde x\Vert^3}]\right\Vert}
 {\mathbb E_{\tilde x\sim\mathbb P}[\frac{1}{\Vert x'-\tilde x\Vert}]}
 =\langle v,\nabla_x f^*(x)|_{x'}\rangle\]
and since for $x\not=0$ and $\Vert w\Vert=1$ it holds $\langle w,x\rangle=\Vert x\Vert\Leftrightarrow w\Vert x\Vert=x$ we discover
$\nabla_x f^*(x)=v\Vert\nabla_x f^*(x)\big|_{x'}\Vert$ and thus
\[\nabla_x f^*(x)\big|_{x'}
=v\Vert\nabla_x f^*(x)\big|_{x'}\Vert
=\frac{\mathbb E_{\tilde x\sim\mathbb P}\left[(\tilde x- x')\frac{f^*(\tilde x)-f^*(x')}{\Vert x'-\tilde x\Vert^3}\right]}
 {\mathbb E_{\tilde x\sim\mathbb P}[\frac{1}{\Vert x'-\tilde x\Vert}]}\]
and with 
\[\nabla_x f^*(x)\big|_{x'}\mathbb E_{\tilde x\sim\mathbb P}[\frac{1}{\Vert x'-\tilde x\Vert}] = \mathbb E_{\tilde x\sim\mathbb P}\left[(\tilde x- x')\frac{f^*(\tilde x)-f^*(x')}{\Vert x'-\tilde x\Vert^3}\right].\]
Plugging this into Eq.\ \ref{E:long_equation} gives us
\[\forall x'\in\supp(\mathbb Q'):\quad \nabla_x\mathbb E_{\tilde x\sim\mathbb P}\left[\frac{f^*(\tilde x)-f^*(x)}{\Vert \tilde x - x\Vert}\right]\bigg|_{x'}=0\]
\end{proof}

\begin{lemma}\label{L:first_order}
 Let $\mathbb P$ and $(\mathbb Q_\theta)_{\theta\in\Theta}$ in $\mathcal P(X)$ and fulfill Assumptions \ref{A:1} and \ref{A:2},
 further let $(\mathbb Q'_\theta)_{\theta\in\Theta}$ be as defined in introduction to Theorem \ref{T:fogan_theorem},
 then for any $f^*\in\oc_{\tau_F}(\mathbb P,\mathbb Q'_\theta)$
 \[\nabla_\theta\tau_F(\mathbb P\Vert\mathbb Q'_\theta)\approx-\frac{1}{2}\nabla_\theta\mathbb E_{x'\sim\mathbb Q'_\theta}[f^*(x')]\]
 thus $f^*$ fulfills Eq.\ \ref{E:first_order_assumption} and $\tau_F$ fulfills Requirement \ref{R:fo_divergence}.
 Further, if $\mathbb P,\mathbb Q_\theta$ are such that there exits an $f$ with $f(x)-f(x')=\Vert x-x'\Vert$ for all $x\in\supp(\mathbb P)$
 and $x'\in\supp(\mathbb Q)$ then
 \[\nabla_\theta\tau_F(\mathbb P\Vert\mathbb Q'_\theta)=-\frac{1}{2}\nabla_\theta\mathbb E_{x'\sim\mathbb Q'_\theta}[f^*(x')]\]
\end{lemma}

\begin{proof}
Start off by noting that for some $f^*\in\oc_{\tau_F}(\mathbb P,\mathbb Q_\theta)$, Theorem 1 from \cite{milgrom2002envelope} gives us
\[\nabla_\theta\tau_F(\mathbb P\Vert\mathbb Q'_\theta)|_{\theta_0}=\nabla_\theta\tau_F(\mathbb P\Vert\mathbb Q'_\theta;f^*)|_{\theta_0}\]
Further, since for $f^*\in\oc_{\tau_F}(\mathbb P,\mathbb Q_\theta)$ it holds
\[\Vert\nabla_x f^*(x)\big|_{x'}\Vert=
  \frac{\left\Vert\mathbb E_{\tilde x\sim\mathbb P}[(\tilde x- x')\frac{f^*(\tilde x)-f^*(x')}{\Vert x'-\tilde x\Vert^3}]\right\Vert}
 {\mathbb E_{\tilde x\sim\mathbb P}[\frac{1}{\Vert x'-\tilde x\Vert}]}\]
the gradient of the gradient penalty part is zero, i.e.
\[\nabla_\theta\mathbb E_{x\sim\mathbb P,x'\sim\mathbb Q_\theta}\left[\left(\Vert\nabla_x f^*(x)\big|_{x'}\Vert-
  \frac{\left\Vert\mathbb E_{\tilde x\sim\mathbb P}[(\tilde x- x')\frac{f^*(\tilde x)-f^*(x')}{\Vert x'-\tilde x\Vert^3}]\right\Vert}
 {\mathbb E_{\tilde x\sim\mathbb P}[\frac{1}{\Vert x'-\tilde x\Vert}]}
  \right)^2\right]=0.\]

One last point needs to be made before the main equation, which is for $x\in\supp(\mathbb P)$
\[\nabla_\theta\mathbb E_{x'\sim\mathbb Q'_\theta}\left[\frac{f^*(x)-f^*(x')}{\Vert x - x'\Vert}\right]\approx 0.\]
This is from the motivation of the penalized Wasserstein GAN where for an optimal critic it should hold that
$f^*(x)-f^*(x')$ is close to $c\Vert x - x'\Vert$ for some constant $c$. Note that if $\mathbb P$ and $\mathbb Q_\theta$
are such that $f^*(x)-f^*(x')=c\Vert x - x'\Vert$ is possible everywhere, then this term is equal to zero.

\begin{align*}
  \nabla_\theta\tau_F(\mathbb P\Vert\mathbb Q'_\theta)|_{\theta_0}
  &=\nabla_\theta\mathbb E_{\mathbb P\otimes\mathbb Q'_{\theta}}[(f^*(x)-f^*(x'))(1-\lambda\frac{f^*(x)-f^*(x')}{\Vert x-x'\Vert})]|_{\theta_0}.
\end{align*}
Since $\mathbb Q_\theta$ fulfills Assumption \ref{A:1}, $\mathbb Q_\theta\sim g(\theta,z)$ where $g$ is differentiable in the first argument and $z\sim\mathbb Z$
($\mathbb Z$ was defined in Assumption \ref{A:1}).
Therefore if we set $g_\theta(\cdot)=g(\theta,\cdot)$ we get
\begin{align}
\nabla_\theta\tau_F(\mathbb P\Vert\mathbb Q')|_{\theta_0}
  \nonumber =\,&\nabla_\theta\mathbb E_{x,\tilde x\sim\mathbb P, z\sim\mathbb Z,\alpha\sim\mathcal U([0,\varepsilon])}\left[(f^*(x)-f^*(\alpha\tilde x + (1-\alpha)g_\theta(z)))\left(1-\lambda\frac{f^*(x)-f^*(\alpha\tilde x + (1-\alpha)g_\theta(z))}{\Vert x-\alpha\tilde x + (1-\alpha)g_\theta(z)\Vert}\right)\right]\bigg|_{\theta_0} \\
  \label{E:big_baddy_1}=\,&-\mathbb E_{x,\tilde x\sim\mathbb P, z\sim\mathbb Z,\alpha\sim\mathcal U([0,\varepsilon])}\left[\nabla_\theta f^*(\alpha\tilde x + (1-\alpha)g_\theta(z))|_{\theta_0}\left(1-\lambda\frac{f^*(x)-f^*(\alpha\tilde x + (1-\alpha)g_{\theta_0}(z))}{\Vert x-\alpha\tilde x + (1-\alpha)g_{\theta_0}(z)\Vert}\right)\right] \\
  \label{E:big_baddy_2}&-\lambda\mathbb E_{x,\tilde x\sim\mathbb P, z\sim\mathbb Z,\alpha\sim\mathcal U([0,\varepsilon])}\left[(f^*(x)-f^*(\alpha\tilde x + (1-\alpha)g_{\theta_0}(z)))\nabla_\theta\left(\frac{f^*(x)-f^*(\alpha\tilde x + (1-\alpha)g_\theta(z))}{\Vert x-\alpha\tilde x + (1-\alpha)g_\theta(z)\Vert}\right)\bigg|_{\theta_0}\right].
  \end{align}
  Now if we look at the \ref{E:big_baddy_1} term of the equation, we see that it's equal to:
  \begin{align*}
  &-\mathbb E_{\tilde x\sim\mathbb P, z\sim\mathbb Z,\alpha\sim\mathcal U([0,\varepsilon])}\left[\nabla_\theta f^*(\alpha\tilde x + (1-\alpha)g_\theta(z))|_{\theta_0}\underbrace{\mathbb E_{x\sim\mathbb P}\left[1-\lambda\frac{f^*(x)-f^*(\alpha\tilde x + (1-\alpha)g_{\theta_0}(z))}{\Vert x-\alpha\tilde x + (1-\alpha)g_{\theta_0}(z)\Vert}\right]}_{=\frac{1}{2}, \text{ Eq.\ \ref{E:slope} from Lemma \ref{L:necessary_for_maximum}}}\right] \\
  =\,&-\frac 1 2 \nabla_\theta\mathbb E_{x'\sim\mathbb Q'_\theta}[f^*(x')]|_{\theta_0} 
  \end{align*}
  and term \ref{E:big_baddy_2} of the equation is equal to
  \begin{align*}
   &-\lambda\mathbb E_{x\sim\mathbb P}\left[f^*(x)\underbrace{\nabla_\theta\mathbb E_{x'\sim\mathbb Q'_\theta}\left[\frac{f^*(x)-f^*(x')}{\Vert x - x'\Vert}\right]}_{\approx 0, \text{ See above}}\bigg|_{\theta_0}\right]\\
   &+\lambda \mathbb E_{\tilde x\sim\mathbb P, z\sim\mathbb Z,\alpha\sim\mathcal U([0,\varepsilon])}\left[f^*(\alpha\tilde x + (1-\alpha)g_{\theta_0}(z))\underbrace{\nabla_\theta\mathbb E_{x\sim\mathbb P}\left[1-\lambda\frac{f^*(x)-f^*(\alpha\tilde x + (1-\alpha)g_{\theta}(z))}{\Vert x-\alpha\tilde x + (1-\alpha)g_{\theta}(z)\Vert}\right]}_{=0, \text{ Eq.\ \ref{E:slope_p_last}}}\bigg|_{\theta_0}\right]
  \end{align*}
  thus showing
  \[\nabla_\theta\tau_F(\mathbb P\Vert\mathbb Q_\theta')|_{\theta_0}\approx-\frac 1 2 \nabla_\theta\mathbb E_{x'\sim\mathbb Q'_\theta}[f^*(x')]|_{\theta_0}\]
\end{proof}

 \begin{lemma}\label{L:wgan_counterexample}
  Let $\tau_I$ be the WGAN-GP divergence defined in Eq.\ \ref{E:WGAN_GP}, let the target distribution be the Dirac distribution $\delta_0$
  and the family of generated distributions be the uniform distributions $\mathcal U([0,\theta])$ with $\theta> 0$. Then there is no
  $C\in\mathbb R$ that fulfills Eq.\ \ref{E:first_order_assumption} for all $\theta > 0$.
 \end{lemma}

 \begin{proof}
  For convenience, we'll restrict ourselves to the $\lambda=1$ case, for $\lambda\not = 1$ the proof is similar. Assume that $f\in\oc_{\tau_I}(\delta_0,\mathcal U([0,\theta])$ and $f(0)=0$. Since $f$ is an optimal critic, for any function $u\in C^1(X)$ and any $\varepsilon\in\mathbb R$
  it holds $\tau_I(\delta_0\Vert\mathcal U([0,\theta]);f)\geq \tau_I(\delta_0\Vert\mathcal U([0,\theta]);f+\varepsilon u)$. Therefore $\varepsilon=0$ is a maximum of the continuously differentiable function
  $\varepsilon\mapsto\tau_I(\delta_0\Vert\mathcal U([0,\theta]);f+\varepsilon u)$, and $\frac{d}{d\varepsilon}\tau_I(\delta_0\Vert\mathcal U([0,\theta]);f+\varepsilon u)|_{\varepsilon=0}=0$. Therefore
  \[\frac{d}{d\varepsilon}\tau_I(\delta_0\Vert\mathcal U([0,\theta]);f+\varepsilon u)|_{\varepsilon=0}=-\int_0^\theta u(t)dt-\int_0^\theta\frac 2 t \int_0^tu'(x)(f'(x)+1)dx\;dt=0\]
  multiplying by $-1$ and deriving with respect to $\theta$ gives us
  \[u(\theta)+\frac 2 \theta \int_0^\theta u'(x)(f'(x)+1)\;dx=0.\]
  Since we already made the assumption that $f(0)=0$ and since $\tau_I(\mathbb P\Vert\mathbb Q;f)=\tau_I(\mathbb P\Vert\mathbb Q;f+c)$ for any constant $c$, we can assume that $u(0)=0$.
  This gives us $u(\theta)=\int_0^\theta u'(x)\;dx$ and thus
  \[\int_0^\theta u'(x)\;dx+\frac 2 \theta \int_0^\theta u'(x)(f'(x)+1)\;dx = \frac 2 \theta \int_0^\theta u'(x)\left(\frac \theta 2 + f'(x)+1\right)\;dx.\]
  Therefore, for the optimal critic it holds $f'(x)=-(\frac\theta 2 + 1)$, and since $f(0)=0$ the optimal critic is $f(x)=-(\frac\theta 2 + 1)x$. Now
  \[\frac{d}{d\theta}\mathbb E_{\mathcal U([0,\theta])}[f]=-\frac{d}{d\theta}\int_0^\theta\left(\frac\theta 2 + 1\right)x\;dx=-\left(\frac\theta 2 + 1\right)\theta\]
  and
  \[\frac{d}{d\theta}\mathbb E_{\delta_0\otimes\mathcal U([0,\theta])}[r_f]=\frac{d}{d\theta}\frac 1 \theta\int_0^\theta\left(\frac\theta 2\right)^2\;dx=\frac{d}{d\theta}\frac{\theta^2}{4}=\frac{\theta}{2}.\]
  Therefore there exists no $\gamma\in\mathbb R$ such that Eq.\ \ref{E:first_order_assumption} holds for every distribution in the WGAN-GP context.
 \end{proof}

 \section{Experiments}

 \subsection{CelebA}\label{SS:appendix_celeba_words}
 The parameters used for CelebA training were:
 \begin{verbatim}
 'batch_size': 64,
 'beta1': 0.5,
 'c_dim': 3,
 'calculate_slope': True,
 'checkpoint_dir': 'logs/1127_220919_.0001_.0001/checkpoints',
 'checkpoint_name': None,
 'counter_start': 0,
 'data_path': 'celebA_cropped/',
 'dataset': 'celebA',
 'discriminator_batch_norm': False,
 'epoch': 81,
 'fid_batch_size': 100,
 'fid_eval_steps': 5000,
 'fid_n_samples': 50000,
 'fid_sample_batchsize': 1000,
 'fid_verbose': True,
 'gan_method': 'penalized_wgan',
 'gradient_penalty': 1.0,
 'incept_path': 'inception-2015-12-05/classify_image_graph_def.pb',
 'input_fname_pattern': '*.jpg',
 'input_height': 64,
 'input_width': None,
 'is_crop': False,
 'is_train': True,
 'learning_rate_d': 0.0001,
 'learning_rate_g': 0.0005,
 'lipschitz_penalty': 0.5,
 'load_checkpoint': False,
 'log_dir': 'logs/0208_191248_.0001_.0005/logs',
 'lr_decay_rate_d': 1.0,
 'lr_decay_rate_g': 1.0,
 'num_discriminator_updates': 1,
 'optimize_penalty': False,
 'output_height': 64,
 'output_width': None,
 'sample_dir': 'logs/0208_191248_.0001_.0005/samples',
 'stats_path': 'stats/fid_stats_celeba.npz',
 'train_size': inf,
 'visualize': False
 \end{verbatim}
 The learned networks (both generator and critic) are then fine-tuned with learning rates divided by 10.
 Samples from the trained model can be viewed in figure \ref{F:fogan_faces}.
 \begin{figure}
  \centering
  \includegraphics[width=.6\textwidth]{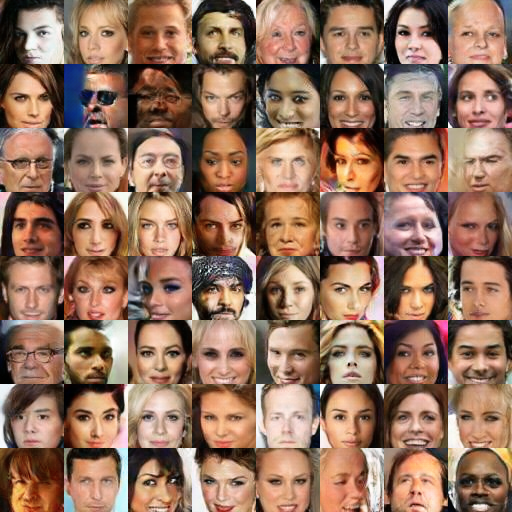}
  \caption{Images from a First Order GAN after training on CelebA data set.}\label{F:fogan_faces}
 \end{figure}
\clearpage

 \subsection{CIFAR-10}\label{SS:appendix_cifar}
 The parameters used for CIFAR-10 training were:
 \begin{verbatim}
  BATCH_SIZE: 64
  BETA1_D: 0.0
  BETA1_G: 0.0
  BETA2_D: 0.9
  BETA2_G: 0.9
  BN_D: True
  BN_G: True
  CHECKPOINT_STEP: 5000
  CRITIC_ITERS: 1
  DATASET: cifar10
  DATA_DIR: /data/cifar10/
  DIM: 32
  D_LR: 0.0003
  FID_BATCH_SIZE: 200
  FID_EVAL_SIZE: 50000
  FID_SAMPLE_BATCH_SIZE: 1000
  FID_STEP: 5000
  GRADIENT_PENALTY: 1.0
  G_LR: 0.0001
  INCEPTION_DIR: /data/inception-2015-12-05
  ITERS: 500000
  ITER_START: 0
  LAMBDA: 10
  LIPSCHITZ_PENALTY: 0.5
  LOAD_CHECKPOINT: False
  LOG_DIR: logs/
  MODE: fogan
  N_GPUS: 1
  OUTPUT_DIM: 3072
  OUTPUT_STEP: 200
  SAMPLES_DIR: /samples
  SAVE_SAMPLES_STEP: 200
  STAT_FILE: /stats/fid_stats_cifar10_train.npz
  TBOARD_DIR: /logs
  TTUR: True
 \end{verbatim}
 The learned networks (both generator and critic) are then fine-tuned with learning rates divided by 10.
 Samples from the trained model can be viewed in figure \ref{F:cifar_samples}.
 \begin{figure}
  \centering
  \includegraphics[width=.6\textwidth]{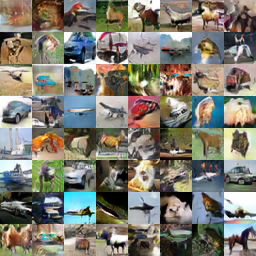}
  \caption{Images from a First Order GAN after training on CIFAR-10 data set.}\label{F:cifar_samples}
 \end{figure}
\clearpage

 \subsection{LSUN}\label{SS:appendix_lsun}
 The parameters used for LSUN Bedrooms training were:
 \begin{verbatim}
 BATCH_SIZE: 64
  BETA1_D: 0.0
  BETA1_G: 0.0
  BETA2_D: 0.9
  BETA2_G: 0.9
  BN_D: True
  BN_G: True
  CHECKPOINT_STEP: 4000
  CRITIC_ITERS: 1
  DATASET: lsun
  DATA_DIR: /data/lsun
  DIM: 64
  D_LR: 0.0003
  FID_BATCH_SIZE: 200
  FID_EVAL_SIZE: 50000
  FID_SAMPLE_BATCH_SIZE: 1000
  FID_STEP: 4000
  GRADIENT_PENALTY: 1.0
  G_LR: 0.0001
  INCEPTION_DIR: /data/inception-2015-12-05
  ITERS: 500000
  ITER_START: 0
  LAMBDA: 10
  LIPSCHITZ_PENALTY: 0.5
  LOAD_CHECKPOINT: False
  LOG_DIR: /logs
  MODE: fogan
  N_GPUS: 1
  OUTPUT_DIM: 12288
  OUTPUT_STEP: 200
  SAMPLES_DIR: /samples
  SAVE_SAMPLES_STEP: 200
  STAT_FILE: /stats/fid_stats_lsun.npz
  TBOARD_DIR: /logs
  TTUR: True
 \end{verbatim}
 The learned networks (both generator and critic) are then fine-tuned with learning rates divided by 10.
 Samples from the trained model can be viewed in figure \ref{F:lsun_samples}.
 \begin{figure}
  \centering
  \includegraphics[width=.6\textwidth]{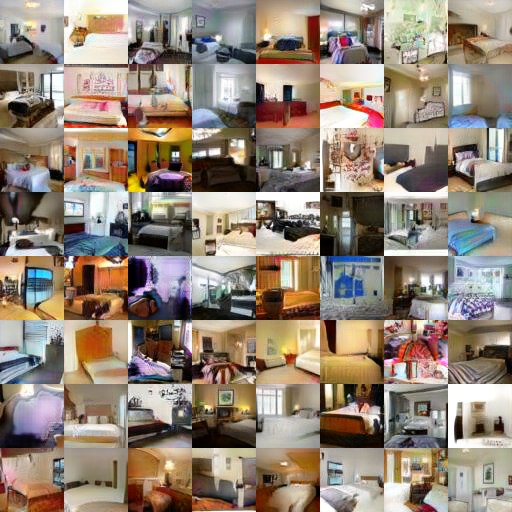}
  \caption{Images from a First Order GAN after training on LSUN data set.}\label{F:lsun_samples}
 \end{figure}
\clearpage
 
 \subsection{Billion Word}\label{SS:appendix_billion_words}
 The parameters used for the Billion Word training were one run with the following settings, followed by
 a second run using initialized with the best saved model from the first run and learning rates divided by 10.
 Samples from our method and the WGAN-GP baseline can be found in figure \ref{F:billion_samples}
 
 \begin{verbatim}
  'activation_d': 'relu',
 'batch_norm_d': False,
 'batch_norm_g': True,
 'batch_size': 64,
 'checkpoint_dir': 'logs/checkpoints/0201_181559_0.000300_0.000100',
 'critic_iters': 1,
 'data_path': '1-billion-word-language-modeling-benchmark-r13output',
 'dim': 512,
 'gan_divergence': 'FOGAN',
 'gradient_penalty': 1.0,
 'is_train': True,
 'iterations': 500000,
 'jsd_test_interval': 2000,
 'learning_rate_d': 0.0003,
 'learning_rate_g': 0.0001,
 'lipschitz_penalty': 0.1,
 'load_checkpoint_dir': 'False',
 'log_dir': 'logs/tboard/0201_181559_0.000300_0.000100',
 'max_n_examples': 10000000,
 'n_ngrams': 6,
 'num_sample_batches': 100,
 'print_interval': 100,
 'sample_dir': 'logs/samples/0201_181559_0.000300_0.000100',
 'seq_len': 32,
 'squared_divergence': False,
 'use_fast_lang_model': True
 \end{verbatim}

\begin{figure}
\centering
\begin{minipage}[b]{.5\linewidth}
\begin{verbatim}
Change spent kands that the righ
Qust of orlists are mave hor int
Is that the spens has lought ant
If a took and their osiy south M
Willing contrased vackering in S
The Ireas's last to vising 5t ..
The FNF sicker , Nalnelber once 
She 's wast to miblue as ganemat
threw pirnatures for hut only a 
Umialasters are not oversup on t
Beacker it this that that that W
Though 's lunge plans wignsper c
He says : WalaMurka in the moroe
\end{verbatim} 
\end{minipage}%
\begin{minipage}[b]{.5\linewidth}
\begin{verbatim}
Dry Hall Sitning tven the concer
There are court phinchs hasffort
He scores a supponied foutver il
Bartfol reportings ane the depor
Seu hid , it ’s watter ’s remold
Later fasted the store the inste
Indiwezal deducated belenseous K
Starfers on Rbama ’s all is lead
Inverdick oper , caldawho ’s non
She said , five by theically rec
RichI , Learly said remain .‘‘‘‘
Reforded live for they were like
The plane was git finally fuels
\end{verbatim} 
\end{minipage}
\caption{Samples generated by First Order GAN trained on fhe One Billion Word benchmark with FOGAN
(left) the original TTUR method (right).}\label{F:billion_samples}
\end{figure}
\end{document}